\documentclass[mathpazo]{csiam-am}
\renewcommand\thisnumber{x}
\renewcommand\thisyear {20xx}
\renewcommand\thismonth{xx}
\renewcommand\thisvolume{x}
\usepackage{xcolor}
\definecolor{chicago-maroon}{RGB}{128,0,0}

\usepackage[margin=1in]{geometry}

\RequirePackage[OT1]{fontenc}
\RequirePackage{amsthm,amsmath,amsfonts,amssymb,mathtools}
\usepackage{bm, bbm}
\usepackage{color}
\usepackage[colorlinks, citecolor=blue, linkcolor=red]{hyperref}
\usepackage{natbib}
\usepackage[capitalise,nameinlink]{cleveref}
\usepackage{subfigure}
\usepackage[version=4]{mhchem}
\usepackage{hyperref}
\usepackage{float}

\newtheorem{assumption}[theorem]{Assumption}

\numberwithin{equation}{section}

\usepackage[linesnumbered,ruled,vlined]{algorithm2e}
\usepackage{color}
\usepackage{graphicx}
\usepackage[section]{placeins}
\usepackage{appendix}

\def\Tr{\mathrm{Tr}}
\def\calH{\mathcal{H}}
\def\I{\mathrm{I}}
\def\op{\texttt{op}}
\def\calN{\mathcal{N}}
\def\bS{\mathbf{S}}
\def\bB{\mathbf{B}}

\def\bD{\mathbf{D}}

\def\bX{\mathbf{X}}
\def\bY{\mathbf{Y}}
\def\bEps{\bm{\mathcal{E}}}

\def\bbR{\mathbb{R}}
\def\bbE{\mathbb{E}}

\def\by{\mathbf{y}}
\def\btheta{\bm{\theta}}
\def\bTheta{\bm{\Theta}}
\def\beps{\bm{\varepsilon}}
\def\eps{\varepsilon}
\def\bx{\mathbf{x}}

\def\bW{\mathbf{W}}
\def\bA{\mathbf{A}}
\def\bw{\mathbf{w}}
\def\bv{\mathbf{v}}

\def\bI{\mathbf{I}}
\def\bS{\mathbf{S}}
\def\bZ{\mathbf{Z}}
\def\bb{\mathbf{b}}

\def\bU{\mathbf{U}}
\def\bM{\mathbf{M}}
\def\bV{\mathbf{V}}
\def\Var{\mathrm{Var}}
\def\our{\texttt{Meta-SP}}

\def\ANIL{\texttt{ANIL}}
\def\col{\text{col}}
\def\hat{\widehat}
\def\tilde{\widetilde}
\def\calA{\mathcal{A}}
\def\calB{\mathcal{B}}

\def\vec{\mathrm{vec}}
\def\sp{\mathsf{sp}}
\def\diff{\mathrm{d}}

\def\F{\texttt{F}}

\def\subG{\mathbf{subG}}

\usepackage{verbatim}

\begin{document}
%%%%% title : short title may not be used but TITLE is required.
\title{\raggedright Few-shot Multi-Task Learning of Linear Invariant Features with Meta Subspace Pursuit}

\author[C. Zhang et~al.]{Chaozhi Zhang\affil{1}, Lin Liu\affil{1,2,3,*}~and Xiaoqun Zhang\affil{1,2}\comma\corrauth}
\address{\affilnum{1}\ School of Mathematical Sciences, Shanghai Jiao Tong University, Shanghai 200240, China\\
\affilnum{2}\ Institute of Natural Sciences, MOE-LSC, Shanghai Jiao Tong University, Shanghai 200240, China\\
\affilnum{3}\ SJTU-Yale Joint Center for Biostatistics and Data Science, Shanghai Jiao Tong University, Shanghai 200240, China}

%
%same address:
%\author[F. Author and A.~Co-Author,]{First Author and A.~Co-Author\corrauth}
%\address{address of First Author and His Best Friend}
%
\emails{{\tt zhangcz4991@sjtu.edu.cn} (C.~Zhang), {\tt linliu@sjtu.edu.cn} (L. Liu), {\tt xqzhang@sjtu.edu.cn} (X.~Zhang)}

%%%%% Begin Abstract %%%%%%%%%%%
\begin{abstract}
Data scarcity poses a serious threat to modern machine learning and artificial intelligence, as their practical success typically relies on the availability of big datasets. One effective strategy to mitigate the issue of insufficient data is to first harness information from other data sources possessing certain similarities in the study design stage, and then employ the multi-task or meta learning framework in the analysis stage. In this paper, we focus on multi-task (or multi-source) linear models whose coefficients across tasks share an invariant low-rank component, a popular structural assumption considered in the recent multi-task or meta learning literature. Under this assumption, we propose a new algorithm, called \texttt{Meta Subspace Pursuit} (abbreviated as \our), that provably learns this invariant subspace shared by different tasks. Under this stylized setup for multi-task or meta learning, we establish both the algorithmic and statistical guarantees of the proposed method. Extensive numerical experiments are conducted, comparing \our{} against several competing methods, including popular, off-the-shelf model-agnostic meta learning algorithms such as \ANIL{}. These experiments demonstrate that \our{} achieves superior performance over the competing methods in various aspects.
\end{abstract}
%%%%% end %%%%%%%%%%%

%%%%% AMS/PACs/Keywords %%%%%%%%%%%
%\pac{}
\ams{62R07, 65F10%The information of the AMS subject classification can be found in http://mathscinet.ams.org/msc/msc2010.html
}

%%%% maketitle %%%%%
\maketitle

%%%% Start %%%%%%

\section{Introduction}
\label{sec:intro}

Modern machine learning methods generally demand a large amount of data to achieve good learning performance. However, in certain application domains such as clinical medicine and social sciences, collecting a large amount of training data can be extremely cost and time consuming, and sometimes even unethical due to privacy concerns. As a result, data scientists and statisticians will have to face the so-called ``data-scarce'' or ``data-hungry'' challenge when deploying those modern machine learning methods \citep{hutchinson2017overcoming, bansal2022systematic}. 

To bypass this difficulty, the multi-task or meta learning paradigm \citep{zhang2021survey, finn2017model, finn2018probabilistic, duan2023adaptive, bairaktari2023multitask, li2023transformers} has emerged as a promising option both in theory and in practice \citep{liu2011multi}, by borrowing information across different data sources/tasks but sharing certain invariances. Intuitively, multi-task or meta learning methods are advantageous because they are designed to acquire those invariant information shared by all tasks during the learning process. 

To make progress in our theoretical understanding of multi-task or meta learning, a large body of the current literature focuses on the multi-task large-dimensional linear models \citep{maurer2016benefit, du2020few, tripuraneni2021provable, tian2025learning}, a stylized statistical model but still preserving some essential features of more general multi-task or meta learning problems encountered in real life. Specifically, we are given a set of $T$ tasks, indexed by $[T] \coloneqq \{1, \cdots, T\}$. For each task $t \in [T]$, we are given access to $m$ independent and identically distributed (i.i.d.) observation pairs $(\bX_{t} \in \bbR^{m \times d}, \by_{t} \in \bbR^{m})$, where $\bX_{t}$ and $\by_{t}$, respectively, denote the design matrix of $d$ features and the scalar-valued responses/labels, over all $m$ samples in task $t$. The total sample size is denoted as $N \equiv T m$. We further assume that for every task $t$, the response $\by_{t}$ follows a noise-corrupted linear model:
\begin{equation}\label{form0}
    \by_t = \bX_t \btheta_t^{\ast} + \beps_t, t = 1, \cdots, T,
\end{equation}
where $\beps_t \in \mathbb{R}^m$ represents the exogenous white noise associated with task $t$. $y_{t, j}, \bx_{t, j}, \eps_{t, j}$ are reserved for the response, the features, and the random noise for sample $j$ in task $t$. In addition, $x_{t, j, k}$ is the $k$-th coordinate of the feature vector $\bx_{t, j}$. We let $\bm{\Theta}^{\ast} \coloneqq [\btheta_{1}^{\ast}, \cdots, \btheta_{T}^{\ast}]^{\top} \in \bbR^{T \times d}$ be the collection of the true regression coefficients over all tasks. Since it becomes more difficult to estimate $\Theta$ when $d$ becomes larger relative to $m$ (corresponding to data-scarcity/data-hungry regime \citep{boursier2022trace}), further structural assumptions on the model \eqref{form0} are indispensable. A popular choice is the so-called ``Hard Parameter Sharing (HPS)'' condition by assuming that the regression coefficients over all tasks enjoy a factor structure $\bm{\Theta}^{\ast} \coloneqq \bW^{\ast} \bB^{\ast}$, where $\bW^{\ast} \in \bbR^{T \times s}$ is the \textit{task-varying} low-rank component with rank $s \leq d$ and $\bB^{\ast} \in \bbR^{s \times d}$ is the \textit{task-invariant} component that maps the intrinsic task-varying coefficient $\bw_{t}^\ast \in \bbR^{s}$ linearly to $\btheta_{t}^\ast \in \bbR^{d}$, for every $t \in [T]$. Intuitively, once the task-invariant $\bB^{\ast}$ is at our disposal, the original high-dimensional task can be reduced to easier, lower-dimensional task by projecting the features onto $\col(\bB^{\ast})$, the linear span of the column space of $\bB^{\ast}$. 

\subsection*{Related Work}

Although exceedingly simple, the model \eqref{form0} has been analyzed extensively as a first-step towards understanding multi-task or meta learning. Due to space limitation, we only highlight several works that are relevant to our development. Tripuraneni et al.\citep{tripuraneni2021provable} introduced a Method-of-Moments (\texttt{MoM}) algorithm that learns the task-invariant subspace $\col(\bB^{\ast})$ via Singular Value Decomposition (SVD) of a $y^{2}$-weighted empirical Gram matrices of the features $\bx$, where we recall that $y$ is the response variable. The \texttt{MoM} method, although extremely simple, is nearly rate optimal for estimating $\col(\bB^{\ast})$ in sine-angle distance by matching the minimax lower bound derived in the same paper \citep{tripuraneni2021provable}, which scales roughly in order $\sqrt{\frac{d s}{T m}}$, where we recall that $d, s, m, T$ are respectively the feature dimension, the rank of the task-invariant subspace $\bB^{\ast}$, the sample size of each individual task, and the total number of tasks. Lower bounds in a similar spirit have also been derived in \cite{bairaktari2023multitask, tian2025learning}, where the authors studied a variant of the HPS linear model by allowing the low-rank subspace/representation $\bB^{\ast}$ to slightly vary across tasks. 

Despite being rate-optimal, as shown later in Section \ref{subsec:simulated}, \texttt{MoM} does not perform well in practice, even in very standard settings where the feature is drawn from isotropic Gaussian. A modified version of \texttt{MoM} in \citep{kong2020meta} offers improved performance in estimating the subspace, particularly when data is extremely scarce. However, this modified \texttt{MoM} struggles to perform that well in practice as the feature dimension and sample size grow (see Figure \ref{fig_m_T=800}). Another alternative method presented in \citep{tripuraneni2021provable} incorporates the Burer-Monteiro factorization method \citep{burer2003nonlinear, lemon2016low, cifuentes2022polynomial}. However, this method fails to effectively handle cases with very scarce data, and the derived rate of estimating $\bB^{\ast}$ does not shrink to zero as the number $T$ of tasks grows, hence a suboptimal rate compared against the minimax lower bound $\sqrt{\frac{d s}{T m}}$. Thekumparampil et al.\citep{thekumparampil2021statistically} proposes an algorithm that involves alternating minimization and gradient descent techniques. Nevertheless, as shown in Figure \ref{fig_m_T=800}, this method still requires a substantial amount of data, which may be difficult to acquire in some applications.

In a slightly different vein, Collins et al.\citep{collins2022maml} demonstrate that Model-Agnostic Meta-Learning (MAML) \citep{finn2017model} and Almost-No-Inner-Loop algorithm (ANIL) \citep{raghu2020rapid, yuksel2023first}, which are not designed for linear models, are still capable of learning the shared representation, or the task-invariant subspace $\bB^{\ast}$. However, the outcomes of these algorithms are not well-suited for data scarcity situations, as their analysis relies on drawing independent mini-batches at each iteration. As shown in our numerical experiments, ANIL requires more data to achieve similar error level (see Figure \ref{fig_mT}). These empirical performance suggests that when the data generating mechanism satisfies certain complexity-reducing structures (e.g. the HPS linear model), MAML and ANIL could be sub-optimal because they fail to leverage such structures at the expense of being model-agnostic.

Finally, we mention that Boursier et al.\citep{boursier2022trace} proposed a trace-norm regularized estimator and derived error bounds of their proposed method\footnote{Their error bounds were later improved to nearly optimal by \cite{liu2023improved} under Gaussianity with Gordon's Convex Gaussian Min-Max Theorem [CGMT] \citep{stojnic2013framework, thrampoulidis2018precise}. However, It further imposes the constraint $T < m$ in \cite{liu2023improved}; see the statement of their Theorem 3.}. As will be evident later, this approach is the closest to ours as the trace norm can be viewed as a relaxation of the matrix rank used in our method. Nevertheless, they do not provide a satisfactory algorithmic solution to address the optimization problem\footnote{In their numerical experiments, the BFGS method is directly used to solve the relevant optimization problems, which may lead to sub-optimal solution. In Section \ref{sec:sim}, we elaborate on this issue when evaluating the finite-sample performance of different methods.}.

\subsection*{A Prelude to Our Approach and Our Contributions}
Overall, the existing methods present various strengths and weaknesses, making it essential to develop new approaches that can effectively handle data scarcity by learning the task-invariant subspace effectively. In this paper, we propose a method of directly learning the shared representation, or equivalently the invariant subspace in the HPS linear model, inspired by the matrix rank minimization problem. We will formally introduce our new algorithm later in Section \ref{sec:alg}. In this section, we provide some intuitive explanation of our proposed algorithm.

Specifically, we address the problem of learning task-invariant subspace by treating it as a matrix rank minimization problem. The corresponding constrained matrix rank minimization problem is formulated as follows:
\begin{equation}\label{intro0}
    \begin{aligned}
     \min_{X} \quad & \text{rank} (X) \\
     \text{s.t.} \quad &  \mathcal{A} (X) = b,
    \end{aligned}
\end{equation}
where $X$ is the unknown matrix, $\mathcal{A}$ is a linear map and $b$ is a vector. As in \citep{lee2009efficient, lee2010admira, jain2010guaranteed}, there is a more robust formulation of \eqref{intro0}:
\begin{equation}\label{intro1}
    \begin{aligned}
     \min_{X} \quad & \Vert \mathcal{A} (X)-b \Vert_2 \\
     \text{s.t.} \quad &  \text{rank} (X) \le s,
    \end{aligned}
\end{equation}
which we will really consider in our methods.

The matrix rank minimization problem is closely related to the compressed sensing problem. They become equivalent when $X$ is reduced to an unknown vector, and the objective function is replaced by $\Vert X \Vert_{0}$. Solving compressed sensing and matrix rank minimization problems exactly, in the form \eqref{intro0}, are well known NP-hard problems \citep{natarajan1995sparse}. To find an approximate solution to the compressed sensing problem, it is common to replace the objective function with the $\ell_1$ norm, which is the convex envelope of the $\ell_0$ norm. Various algorithms, such as iterative hard thresholding \citep{blumensath2009iterative}, subspace pursuit \citep{dai2009subspace}, compressive sampling matching pursuit \citep{needell2009cosamp}, iterative thresholding with inversion \citep{maleki2009coherence}, and hard thresholding pursuit \citep{foucart2011hard}, are typically employed for the original $\ell_0$ minimization problem. For the $\ell_1$ optimization problem, several algorithms, such as those introduced in \citep{candes2005magic,figueiredo2007gradient,hale2008fixed,yin2008bregman}, are commonly used.

In analogy, nuclear norm minimization, which replaces the rank of the matrix with its nuclear norm, solves the matrix rank minimization problem \eqref{intro0} approximately. The method proposed in \citep{boursier2022trace} exemplifies this approach. As the nuclear norm minimization problem is equivalent to some semi-definite programming (SDP) problem, interior-point methods can be used to solve it, as shown in \citep{liu2010interior}. Several other first-order algorithms, including singular value thresholding (SVT), fixed point continuation (FPC), and the approximate SVD-based FPC algorithm (FPCA), have also been proposed in \citep{cai2010singular} and \citep{ma2011fixed}. In our paper, we opt to utilize the iterative hard thresholding algorithm mentioned in \citep{goldfarb2011convergence}. This algorithm is a generalization of the iterative hard thresholding algorithm used in compressed sensing \citep{blumensath2009iterative, ma2013sparse, carpentier2018iterative}. Within the context of meta-learning, we designate this algorithm as \texttt{Meta Subspace Pursuit} (\our{}). We substantiate its efficacy by demonstrating that it can attain a convergence outcome analogous to the one reported in \citep{boursier2022trace}. In practical applications, our approach demonstrates improved statistical properties and computational efficiency empirically.

We now summarize our main contributions as follows:
\begin{enumerate}
\item On the methodological side, we develop a new, iterative algorithm \our{} for learning the task-invariant subspace in the multi-task HPS linear models, contributing a new method to the multi-task/meta/invariant learning paradigm \citep{rojas2018invariant}. \our{} is easy and transparent to implement.
\item On the theoretical side, we establish how fast the regression coefficients $\bTheta$ and task-invariant subspace $\bB$ outputted by \our{} converge to the truth $\bTheta^{\ast}$ and $\bB^{\ast}$, as $m, d, s$ vary. In particular, we directly analyze the iterations of \our{} and establish the convergence rates at $k$-th iteration, for every $k \geq 1$. As a consequence, our analysis directly reveals the iteration complexity of \our{}.
\item Empirically, through extensive experiments based on simulated and real datasets, we demonstrate that \our{} outperforms most of the competing methods on several aspects. Several future directions for theoretical investigation are also hinted by the results of our numerical experiments. For more details, see Section \ref{sec:sim}. 
\end{enumerate}

\subsection*{Organization of the Paper}
The rest of our paper is organized as follows. In Section \ref{sec:problem}, we introduce the basic problem setup, the statistical model being analyzed, and key underlying assumptions. The new \our{} algorithm is then proposed in Section \ref{sec:method}, with its theoretical properties, in terms of sample and computational complexities, established in Section \ref{sec:theory}. In Section \ref{sec:sim}, we conduct numerical experiments with both simulated and real datasets that demonstrate the performance of our proposed method in practice, together with an extensive comparison between our method and other competing methods. Finally, we conclude our paper in Section \ref{sec:conclusion} with a discussion on future research directions.

\subsection*{Notation}

Before proceeding, we collect some notation frequently used throughout the paper. The distributions for multivariate Gaussian and sub-Gaussian random variables are denoted as $\calN (\bm{\mu},\bm{\Sigma})$ and $\subG (\bm{\mu},\bm{\Sigma})$, respectively, where $\bm{\mu}$ is the mean and $\bm{\Sigma}$ is the covariance matrix. For sub-Gaussian random variables $\bX$, by $\subG (\bm{\mu}, \bm{\Sigma})$, we mean that the population covariance matrix $\bbE [(\bX - \bm{\mu}) (\bX - \bm{\mu})^{\top}]$ is dominated by $\bm{\Sigma}$ in the positive semidefinite sense. In general, we reserve bold fonts for vectors and matrices and regular fonts for their elements.

Given a real-valued vector $\mathbf{b}$, let $\|\mathbf{b}\|_q$ denote its $\ell_q$ norms. For a real-valued matrix $\bB$, $\|\bB\|_2$, $\|\bB\|_{\F}$, $\|\bB\|_{\ast}$, and $\text{rank}(\bB)$ represent its spectral norm, Frobenius norm, trace norm, and rank respectively. $\col (\bB)$ and $\col^{\perp} (\bB)$ denote the column space spanned by $\bB$ and its orthocomplement. With slight abuse of notation, we denote $\bB^{\perp}$ as a matrix with $\col (\bB^{\perp}) = \col^{\perp} (\bB)$. In general, $\bB^{\perp}$ is not unique, and we just pick one arbitrarily. Throughout the paper, we use the sine angle distance to measure the distance between the column spaces spanned by two different matrices $\bB_{1}$ and $\bB_{2}$, defined as $\sin {\angle (\bB_{1}, \bB_{2})} \coloneqq \Vert \bB_{1}^{\top} \bB_{2}^{\perp} \Vert_{2}$ \citep{stewart1990matrix}. Given any positive integer $k$, we let $\I_{k}$ denote the $k \times k$ identity matrix. Given a matrix $\bM \in \bbR^{k_1 \times k_2}$, we introduce the following vectorization operator to flatten the matrix $\bM$:
\begin{align*}
\vec (\bM) \coloneqq \left[ \mathbf{m}_{1, \cdot}^{\top}, \cdots, \mathbf{m}_{k_{1}, \cdot}^{\top} \right]^{\top}.
\end{align*}
If $\bM$ is a symmetric positive semidefinite matrix, let $\bM^{\dag}$ denote its Moore-Penrose pseudo inverse (also known as the ridgeless regularized inverse $\bM^{\dag} \coloneqq \lim_{\lambda \downarrow 0} (\bM + \lambda \bI)^{-1}$): let $\bM = \bU \bD \bU^{\top}$ be the eigen-decomposition of $\bM$ and let $\bar{\bD}$ be the diagonal matrix of only the non-zero eigenvalues and $\bar{\bU}$ be the corresponding eigenvectors, then $\bM^{\dag} \equiv \bar{\bU} \bar{\bD}^{-1} \bar{\bU}$. Finally, given two matrices $\bA$ and $\bB$ of appropriate sizes, $\bA \odot \bB$ denotes their matrix Hadamard product.

\section{Problem setup and main assumptions}
\label{sec:problem}

\subsection{The statistical model}
\label{sec:model}

As alluded to in the Introduction, the Hard Parameter Sharing (HPS) linear model posits a factor structure of the regression coefficients among different tasks: There exists a task-invariant matrix $\bB^{\ast} \in \bbR^{s \times d}$ and a task-varying, but lower-dimensional matrix $\bW^{\ast} \in \bbR^{T \times s}$, with $s \leq d$, such that $\bm{\Theta}^{\ast} = \bW^{\ast} \bB^{\ast}$. Here $\bW^{\ast} \equiv [\bw_1^{\ast}, \cdots, \bw_T^{\ast}]^{\top}$ can be interpreted as the intrinsic task-specific coefficients and $\bB^{\ast}$ represents the task-invariant linear map that lifts $\bw_{t} \in \bbR^{s}$ to $\btheta_{t} \in \bbR^{d}$ for all $t \in [T]$. The matrix $\bB^{\ast}$ is sometimes called the task-invariant linear representation in the literature \citep{tripuraneni2021provable}.The matrix $\bB^{\ast}$ is assumed to have orthonormal columns as we are only interested in the linear space $\col(\bB^{\ast})$ spanned by the columns of $\bB^{\ast}$. To avoid clutter, $\bB^{\ast}$ will also mean $\col(\bB^{\ast})$ when it is clear from the context.

In multi-task or meta learning, an essential idea is to learn the low-dimensional shared representation $\bB^{\ast}$ by $\hat{\bB}$ trained from multiple data sources. Based on learnt $\hat{\bB}$, the original high-dimensional problem is reduced to lower dimensions, thus enjoying improved sample and computational complexity. Following the current literature \citep{boursier2022trace, collins2022maml, tripuraneni2021provable}, the quality of learning $\bB^{\ast}$ by $\hat{\bB}$ is gauged by the sine angle distance $\sin {\angle(\hat{\bB}, \bB^{\ast})}$.

With Model \eqref{form0} in place, we also need to specify some assumptions on the data. First, we impose the following distributional assumptions on the covariates $\bX_{t}$ and the additive noise $\beps_{t}$:
\begin{assumption}\label{assum1}
For every task $t \in [T]$, $x_{t, j, k} \overset{\text{i.i.d.}}{\sim} \subG (0, 1)$ for $j = 1, \cdots, m$ and $k = 1, \cdots, d$; and $\varepsilon_{t, j} \sim \subG (0, \sigma^{2})$ for $j = 1, \cdots, m$.
\end{assumption}

The distributional assumption imposed on $\bX_{t}$ implies that it satisfies the Restricted Isometry Property (RIP) (or properties alike such as the Restricted Eigenvalue Conditions \citep{van2009conditions, buhlmann2011statistics}) with high probability, a common condition imposed on the covariates to ensure identifiability (i.e. well-posedness along any sparse directions). Similar to the existing literature, we also assume the following condition on the intrinsic task-specific regression coefficients $\bW^{\ast}$:
\begin{assumption}[Task Diversity]\label{assum2}
Let $\lambda_{1}, \cdots, \lambda_{s}$ be the largest to the smallest eigenvalues of the Gram matrix of task-varying components averaged over tasks, $\bm{\Xi}^{\ast} \coloneqq T^{-1} \bW^{\ast \top} \bW^{\ast} \equiv T^{-1} \sum_{t = 1}^{T} \bw_{t}^{\ast} \bw_{t}^{\ast \top}$. There exists some universal constant $L_{s} > 0$ such that $\lambda_{s} \geq L_{s}$.
\end{assumption}
This assumption essentially requires that the task-varying information in the regression coefficients is sufficiently diverse across tasks, to the extent that their Gram matrix $\bm{\Xi}^{\ast}$ averaged across tasks is well-conditioned.

To ease exposition, we introduce some further notations: $\mathbb{X} \in \bbR^{mT \times dT}$ denotes the block-diagonal matrix with block-diagonal elements $\bX_{t} \in \bbR^{m \times d}$ for $t = 1, \cdots, T$, $\bY = (\by_1^{\top}, \cdots, \by_T^{\top})^{\top} \in \bbR^{N \times 1}$ and $\bEps = (\beps_1^{\top}, \cdots, \beps_T^{\top})^{\top} \in \bbR^{N \times 1}$, with $N \coloneqq m T$. Additionally, given any $\bm{\Theta} \in \bbR^{T \times d}$, we define the linear mapping $\calA: \bbR^{T \times d} \rightarrow \bbR^{N \times 1}$ as follows
\begin{equation}\label{form1}
\calA (\bm\Theta) \coloneqq \mathbb{X} \cdot \vec (\bm\Theta) \equiv \left[ \begin{array}{ccc}
    \bX_1 & & \\
    & \ddots & \\
    & & \bX_T \\
    \end{array} \right]
    \left[ {\begin{array}{c}
    \btheta_1 \\
    \vdots \\
    \btheta_T \\
    \end{array} } \right].
\end{equation}
Hence we can rewrite Model \eqref{form0} compactly in matrix form as:
\begin{align*}
\bY = \calA (\bm\Theta^{\ast}) + \bEps.
\end{align*}
Our formulation can also be adapted to the case where different tasks have different sample sizes $m_{1}, \cdots, m_{T}$. But to not lead readers astray, we decide to focus on the simpler, equal-sized case, i.e. $m_{1} \equiv \cdots \equiv m_{T} \equiv m$.

\subsection{Restricted Isometry Property}

Similar to the compressed sensing literature, In the pursuit of solving the aforementioned problem through matrix rank minimization, the concept of RIP plays a pivotal role. While RIP for sparse vectors in compressed sensing was originally introduced in \cite{candes2005decoding}, it has been extended to matrices to address matrix rank minimization, as seen in \cite{recht2010guaranteed}.  Here, we define RIP in the context of Model \eqref{form0}:
\begin{definition}[Restricted Isometry Property]\label{RIP}
For any integer $r$ where $1\le r \le d$, consider the linear operator $\calB:\mathbb{R}^{T \times d} \rightarrow \mathbb{R}^{N \times 1}$. The operator $\calB$ satisfies the Restricted Isometry Property with the restricted isometry
constant $\delta_r(\calB)$, where $\delta_r(\calB)$ is the smallest constant for which the following inequality holds:
\begin{equation}\label{rip}
    (1-\delta_r(\calB))\|\bm\Theta\|_{\F}^2 \le \|\calB(\bm\Theta)\|_2^2 \le (1+\delta_r(\calB))\|\bm\Theta\|_{\F}^2
\end{equation}
for all $\bm\Theta\in\mathbb{R}^{T \times d}$ with $\text{rank}(\bm\Theta)\le r$.
\end{definition}

In our context, we define the linear operator $\calB=\mathcal{A}/\sqrt{m}$, where $\mathcal{A}$ is as defined in Equation \eqref{form1}. Equation \eqref{rip} can be rewritten as:
\begin{equation}\label{rip1}
    (1-\delta_r(\calB))(\sum_{t=1}^T\|\btheta_t\|_2^2) \le \sum_{t=1}^T\|\frac{1}{\sqrt{m}}\bX_t\btheta_t\|_2^2 \le (1+\delta_r(\calB))(\sum_{t=1}^T\|\btheta_t\|_2^2)
\end{equation}
Since $\text{rank}(\bm\Theta)\le r$, all $\btheta_t$ lie within an $r$-dimensional subspace. Therefore, there exist an orthonormal matrix $\bB\in\mathbb{R}^{r \times d}$($\bB\bB^{\top}=\bI_r$) and a set of vectors $\{\bv_t\}_{t=1}^{T}$ such that $\btheta_t=\bB^{\top}\bv_t$. Equation \eqref{rip1} can then be equivalently expressed as:
\begin{equation}\label{rip2}
    (1-\delta_r(\calB))(\sum_{t=1}^T\bv_t^{\top}\bv_t)
    % = (1-\delta_r(\calB))(\sum_{t=1}^T\|\bB\bv_t\|_2^2) 
    \le \sum_{t=1}^T\frac{1}{m}\bv_t^{\top}\bB\bX_t^{\top}\bX_t\bB^{\top}\bv_t \le
    % (1+\delta_r(\calB))(\sum_{t=1}^T\|\bB\bv_t\|_2^2) = 
    (1+\delta_r(\calB))(\sum_{t=1}^T\bv_t^{\top}\bv_t).
\end{equation}
This version of the RIP condition can often be justified by the distributional assumptions on the feature vector $\bx$, using matrix concentration inequalities \citep{tropp2015introduction} (see Appendix \ref{app-proof-ripr}). Then we obtain the following:

\begin{theorem}\label{ripr}
The linear operator $\calB=\mathcal{A}/\sqrt{m}$ where $\mathcal{A}$ is defined in Equation \eqref{form1}, satisfies the Restricted Isometry Property with the restricted isometry constant:
\begin{equation*}
    \delta_r(\calB)\le \sqrt{\frac{8(a+1)r-4}{3m}\log{\frac{2r}{\epsilon}}}
\end{equation*}
with probability at least $(1-\epsilon)^T(1-\frac{15}{a^2})^{rT}$ for any $a>0$.
\end{theorem}

The proof of Theorem \ref{ripr} is presented in Appendix \ref{app-proof-ripr}. It demonstrates that when the data generating mechanism follows Assumption \ref{assum1}, the associated linear operator $\mathcal{B}$ satisfies the RIP condition with high probability. This guarantee forms the building block for the convergence analysis of our proposed method.

\section{Our method}
\label{sec:method}

In this section, we are ready to describe the new methodology for learning the invariant subspace $\bB^{\ast}$ under the HPS linear model. The method is coined as the \texttt{Meta Subspace Pursuit} Algorithm, or \our{} for short.
%\subsection{Solving the Matrix Rank Minimization Problem}

In a nutshell, \our{} consists of two steps. In the first step, task-specific regressions are solved jointly to obtain the coefficient matrix estimator $\hat{\bm\Theta}$. To simplify exposition, we assume that the rank $s$ of $\bm\Theta^\ast$ (or equivalently, the dimension $s$ of the subspace $\bB$) is known, whence $\hat{\bm{\Theta}}$ is the solution to the following constrained optimization problem:
\begin{equation}\label{min0}
    \begin{aligned}
     \min_{\bm\Theta} \quad &\mathcal{L}(\bm\Theta):=\|\mathcal{A}(\bTheta)-\bY\|_{\texttt{F}}^2,\\
     \text{s.t.} \quad & \text{rank}(\bTheta) \le s.
    \end{aligned}
\end{equation}
Here, $\mathcal{A} (\bm\Theta)$ is the linear mapping defined in Equation \eqref{form1}. 

In the second step, with $\bm\Theta$ just attained from step one, it is then straightforward to apply either QR matrix factorization or SVD to obtain the task-invariant subspace $\bm{B}$. In the following, we introduce the proposed \texttt{Meta Subspace Pursuit} (\our{}) algorithm for solving this constrained optimization problem. 

\subsection{The \texttt{Meta Subspace Pursuit} (\our{}) Algorithm}
\label{sec:alg}

The \our{} algorithm can be viewed as a variant of the iterative hard thresholding (IHT) algorithm, initially introduced for solving compressed sensing problems in \citep{blumensath2009iterative}. The IHT algorithm has since then been adapted and generalized to handle matrix rank minimization problems, similar to the formulation in Equation \eqref{intro1}. Here we further adapt the matrix rank minimization problem to solve the task-invariant subspace learning problem. Since \eqref{min0} and \eqref{intro1} have the same form when $\mathcal{A}(\bTheta)$ and $\bY$ are vectorized, we can use this kind of algorithm to solve \eqref{min0}.

The \our{} algorithm constitutes the following iterative subroutines:
\begin{enumerate}
    \item \textbf{Gradient Descent (GD) Step}. In each iteration, a GD update is performed independently for each task to reduce the task-specific loss. For the $t$-th task, the expression of $(k + 1)$-th iteration is:
    \begin{equation}
       \label{update}
        \hat{\btheta}_t^{(k+1)} = \btheta_t^{(k)} - \nabla_{\btheta_t} \mathcal{L} (\bTheta) |_{\btheta_t = \btheta_t^{(k)}} = \btheta_t^{(k)} + \frac{\gamma}{m} \bX_t^{\top} (\by_t - \bX_t \btheta_t^{(k)})
    \end{equation}
    where $\gamma$ is the step size and $\btheta_t^{(k)}$ is the regression coefficients updated after the $k$-th step. The updated regression coefficients are $\hat{\bm\Theta}^{(k+1)} \coloneqq [\hat{\btheta}_1^{(k+1)}, \cdots, \hat{\btheta}_T^{(k+1)}]^{\top}$.
    
    \item \textbf{Hard Thresholding (HT) Step}. We use SVD to factorize $\hat{\bm\Theta}^{(k+1)}=\hat{\bU}^{(k+1)} \hat{\bD}^{(k+1)} {{}\hat{\bV}^{(k+1)}}^{\top}$, where $\hat{\bU}^{(k+1)} \in \mathbb{R}^{T \times T}$, $\hat{\bD}^{(k+1)} \in \mathbb{R}^{T \times d}$ and $\hat{\bV}^{(k+1)} \in \mathbb{R}^{d \times d}$. Then, only the largest $s$ singular values and the corresponding left and right singular vectors are retained, where we recall that $s$ is the task-invariant subspace dimension. We obtain a new matrix $\bm\Theta^{(k + 1)} = \bU^{(k + 1)} \bD^{(k + 1)} {\bV^{(k + 1)}}^{\top}$, where $\bU^{(k+1)} \in \mathbb{R}^{T \times s}$, $\bV^{(k + 1)} \in \mathbb{R}^{d \times s}$ and $\hat{\bD}^{(k + 1)} \in \mathbb{R}^{s \times s}$. For short, we denote this hard thresholding operation as $\calH_s (\cdot)$, so $\calH_s (\hat{\bm\Theta}^{(k + 1)}) = \bm\Theta^{(k + 1)}$.
    
    \item \textbf{Representation Learning Update}. The right singular vectors $\bV^{(k + 1) \top}$ is taken as the learnt task-invariant subspace $\bB^{(k + 1)}$ at the $(k+1)$-th iteration.
\end{enumerate}
The \our{} algorithm iteratively performs the above three steps until a stopping criterion is met. The complete procedure of the \our{} is presented in Algorithm \ref{IHT}. It is noteworthy that the step size $\gamma$ is an important hyperparameter -- the algorithm converges to the desired solution only if $\gamma$ is set appropriately.

In the \our{} algorithm, the choice of the parameter $s$ in the hard thresholding step is an important hyperparameter. In the ideal case where the dimension of the shared representation among tasks is known, setting $s$ to that value is sufficient. However, in realistic scenarios, the true dimension of the shared representation is generally unknown. In such cases, the algorithm can be executed on a subset of samples in each task using different values of $s$, and the most appropriate value can then be selected based on the generalization performance of the learned shared representation on other tasks.

\begin{algorithm}[ht]
\caption{\texttt{Meta Subspace Pursuit} (\our{})}
\label{IHT}
    \SetAlFnt
    \SetAlgoLined
    \SetKwData{Left}{left}\SetKwData{This}{this}\SetKwData{Up}{up}\SetKwFunction{Union}{Union}\SetKwFunction{FindCompress}{FindCompress}\SetKwInOut{Input}{Input}\SetKwInOut{Output}{Output}
    Initialize $k = 0$ and $\bm\Theta^{(0)}$ arbitrarily; Set step size $\gamma$\;
    \While{not done}{
        Gradient Descent: $\hat{\bm\Theta}^{(k+1)} \coloneqq [\hat{\btheta}_1^{(k+1)},\cdots,\hat{\btheta}_T^{(k+1)}]^{\top}$, where $\hat{\btheta}_t^{(k+1)} = \btheta_t^{(k)} + \dfrac{\gamma}{m} \bX_t^{\top} (\by_t - \bX_t \btheta_t^{(k)})$\;
        Hard Thresholding: $\bm\Theta^{(k+1)} = \calH_s (\hat{\bm\Theta}^{(k+1)})$, i.e. $\bm\Theta^{(k+1)} = \bU^{(k+1)} \bD^{(k+1)} \bV^{(k+1) \top}$\;
        Representation Learning Update: $\bB^{(k+1)} = \bV^{(k+1) \top}$\;
        $k \leftarrow k+1$\;
    }
\end{algorithm}

\subsection{Theoretical analysis of \our{}}
\label{sec:theory}

In this section, our goal is to establish error bounds for the estimator $\bm\Theta^{(k)}$ obtained by the \our{} algorithm and the  ground truth $\bm\Theta^{\ast}$. Additionally, we aim to constrain the sine angle distance between the subspaces $\mathbf{B}^{(k)}$ and $\bB^{\ast}$.

In this context, we operate under the assumption that RIP is valid, as established in Theorem \ref{ripr}. This assumption provides us the leverage to establish error bounds for the $\bm\Theta^{(k)}$ estimator. Our primary theorem is presented below.

\begin{theorem}\label{thm-theta}
Under Assumptions \ref{assum1} to \ref{assum2}, we suppose the restricted isometry constant of $\calB=\mathcal{A}/\sqrt{m}$ is $\delta_s(\calB)$($\mathcal{A}$ is defined in Equation \eqref{form1} and $\delta_s(\calB)$ satisfies Theorem \ref{ripr}). If $\frac{1}{2(1-2\delta_{3s}(\calB))}<\gamma<\frac{1}{1-2\delta_{3s}(\calB)}$, for any $a > 0$, with probability at least $(1 - \epsilon) \left(1 - \frac{15}{a^2} \right)^{3s} \left( 1 - \frac{2}{d} - \frac{14}{m} - \frac{208}{dm} \right)^T$, the following holds:
\begin{equation}\label{res1-0}
\|\bm\Theta^{\ast} - \bm\Theta^{(k)}\|_{\texttt{F}} \le \left[2(1-\gamma+2\gamma\delta_{3s}(\calB))\right]^k\|\bm\Theta^{\ast}-\bm\Theta^{(0)}\|_{\F} + \frac{2\sqrt{2}}{1-2(1-\gamma+2\gamma\delta_{3s}(\calB))}\sqrt{\frac{dT\sigma^2}{m}}.
\end{equation}

Moreover, as $k \rightarrow \infty$, the above bounds can be simplified to:
\begin{equation}\label{res1-2}
    \|\bm\Theta^{\ast}-\bm\Theta^{(k)}\|_{\texttt{F}} \le O(\sqrt{\frac{\sigma^2dT}{m}})
\end{equation}
\end{theorem}
The proof of Theorem \ref{thm-theta} is deferred to Appendix \ref{app-proof-thm-theta}. Our proof loosely follows the strategy of \cite{goldfarb2011convergence}. However, during the proof, we need to derive more refined error bounds depending on both the sample size $m$ and the data dimension $d$. 

The bound in Frobenius norm shown in Theorem \ref{thm-theta} can then be used to establish convergence rate of the learnt task-invariant subspace $\bB^{(k)}$ after $k$-th iteration of \our{} to the truth $\bB^{\ast}$.

\begin{theorem}\label{thm-B}
Under the same assumptions of Theorem \ref{thm-theta}, we have:
\begin{equation}\label{res2-1}
    \sin{\angle (\bB^{(k)}, \bB^{\ast})} \le \frac{\|\bm\Theta^{\ast}-\bm\Theta^{(k)}\|_{\texttt{F}}}{\sqrt{L_{s} T}}.
\end{equation}
Furthermore, with the error bound for $\|\bm\Theta^{\ast} - \bm\Theta^{(k)}\|_{\texttt{F}}$ shown in \eqref{res1-0}, i.e. with probability at least $(1 - \epsilon) (1 - \frac{15}{a^2})^{3s} ( 1 - \frac{2}{d} - \frac{14}{m} - \frac{208}{dm})^T$, we have 
\begin{equation}\label{res2-2}
    \sin{\angle (\bB^{(k)}, \bB^{\ast})} \le O \left( \sqrt{\frac{\sigma^2d}{L_{s}m}} \right).
\end{equation}
\end{theorem}

The proof of Theorem \ref{thm-B} is provided in Appendix \ref{app-proof-thm-B}.

\subsubsection{Remarks on the Main Theoretical Results}
Theorem \ref{thm-theta} provides valuable insights into the error bound of $\bm\Theta^{(k)}$ and $\bm\Theta^{\ast}$. This result places significant reliance on RIP, which in turn necessitates certain conditions to be met. Specifically, to ensure that the restricted isometry constant in Theorem \ref{ripr} remains small, the value of $m$ should not be small($m \ge \Omega (s \log s)$). Additionally, the distance of the initial point should not be excessively large. In practice, the choice of the zero point is often suitable and aligns with the selection made in the subsequent numerical experiments. Importantly, this error bound diminishes as $m$ increases.

In practice, we often focus more on the distance metric $\sin{\angle(\bB^{(k)},\bB^{\ast})}$. Theorem \ref{thm-B} demonstrates that this distance can be bounded by the distance presented in Theorem \ref{thm-theta}, and it also tends to decrease as $m$ increases. Notably, this distance should also decrease with an increase in $T$, as observed in the numerical experiments in next section. While it has been empirically observed, further theoretical validation remains a topic for future research.

The theoretical results depend critically on the matrix RIP, which itself arises from the assumptions introduced in Section \ref{sec:model}. We acknowledge that these assumptions are idealized and may be difficult to verify in real-world multi-task problems. However, the experiments on real data (see Section \ref{subsec: real}) demonstrate that our method remains stable and achieves competitive performance, indicating that it captures meaningful low-rank structures and possesses practical significance within the multi-task learning paradigm.

Finally, we comment on how the theoretical results scale with sample size per task $m$ and compare the scaling with that of several other algorithms. Recalling that we require $m \ge \Omega(s \log s)$, which results in a smaller $m$ requirement compared to other methods. For instance, Burer-Monteiro factorization in \cite{tripuraneni2021provable} necessitates $m\ge\Omega(s^4\log(T))$, and alternating minimization in \cite{thekumparampil2021statistically} requires $m\ge\Omega(s^2\log(T))$. The Method of Moments in \cite{tripuraneni2021provable} indeed shares a similar requirement with our approach, where it requires $m \ge \Omega (s \log s)$. The error bound of the obtained $\bm\Theta$ for this method, as specified in Theorem \ref{thm-theta}, is $O(\sigma\sqrt{s\frac{\sigma^2sd+T}{m}}+s\sqrt{\frac{d}{m}})$. This error bound may not be as favorable, as evident from the numerical results presented in the subsequent section. Another method, nuclear norm minimization in \cite{boursier2022trace}, has been shown to require $m\ge\Omega(1)$, which is favorable in theory but may not perform as well in practical scenarios, especially when $s$ is small. The error bound of the obtained $\bm\Theta$ for this method, as specified in Theorem \ref{thm-theta}, is $O \left( \sigma \sqrt{s \left( \frac{d^{2}}{m^{2}} + \frac{T}{m} \right)}+\sqrt{s \frac{d}{m} \frac{\max \{d, T\}}{m}} \right)$. This bound is less favorable when the noise level $\sigma$ is small. In summary, our method \our{} holds a distinct advantage when dealing with scenarios where the sample complexity per task $m$ is small.

\section{Numerical Results}
\label{sec:sim}

\subsection{Simulated Data}
\label{subsec:simulated}

In this section, we conduct extensive numerical experiments to evaluate our proposed method, benchmarked against several competing methods, including both the classical model-based methods such as Method of Moments (\texttt{MoM}) and the modern model-agnostic methods such as $\ANIL$. Specifically, we consider a setup where $d = 100$ and $s = 5$, following the experimental setup outlined in \cite{tripuraneni2021provable,thekumparampil2021statistically,boursier2022trace}. We use the following performance metric to compare different methods: (1) the normalized squared Frobenius norm distance
\begin{equation*}
\mathbf{Dist}_1 (\bm\Theta^\dag, \bm\Theta^{\ast}) \coloneqq \|\bm\Theta^\dag - \bm\Theta^{\ast}\|_{\texttt{F}}^2/T
\end{equation*}
and (2) the sine angle distance
\begin{equation*}
\mathbf{Dist}_2 (\bB^\dag, \bB^{\ast}) \coloneqq \sin{\angle(\bB^\dag, \bB^{\ast})}
\end{equation*}
where $\bm\Theta$ and $\bB$, respectively, denote arbitrary values that $\bm\Theta$ and $\bB$ can take and $\bm\Theta^\ast$ and $\bB^\ast$ denote the ground truths, respectively. Using repeating draws from the true data generating process, we investigate on average how these performance metrics are affected by key parameters including the task-specific sample size $m$, the total number of tasks $T$, and the noise level $\sigma$. Additionally, we examine their dynamics changing with iteration steps and system running time in a single experiment. More details on the experimental setups can be found in the Appendix.

We compare our method \our{} with several alternative approaches:
\begin{itemize}
    \item \texttt{AltMin} and \texttt{AltMinGD}: The alternating minimization and alternating minimization gradient descent algorithms studied in \cite{thekumparampil2021statistically}.
    \item \texttt{BM}: The Burer-Monteiro factorization method as presented in \cite{tripuraneni2021provable}.
    \item \texttt{MoM}: The Method of Moments algorithm described in \cite{tripuraneni2021provable}.
    \item \texttt{MoM2}: Algorithm 2 presented in \cite{kong2020meta}.
    \item \texttt{NUC}: The nuclear norm minimization method proposed in \cite{boursier2022trace}.
    \item \texttt{ANIL}: The ANIL algorithm outlined in \cite{collins2022maml}, which utilizes all available data in each iteration.
    \item \our{}: Our \texttt{Meta Subspace Pursuit} method.
\end{itemize}
\texttt{ANIL} was not originally designed for the linear model considered in this work, although it was applied to this model in \cite{collins2022maml}. Here, we include \texttt{ANIL} as an illustrative comparison to show the difference between low-rank approximation–based methods and meta-learning approaches.

\begin{figure}[ht]
    \centering
    \includegraphics[width=.45\textwidth]{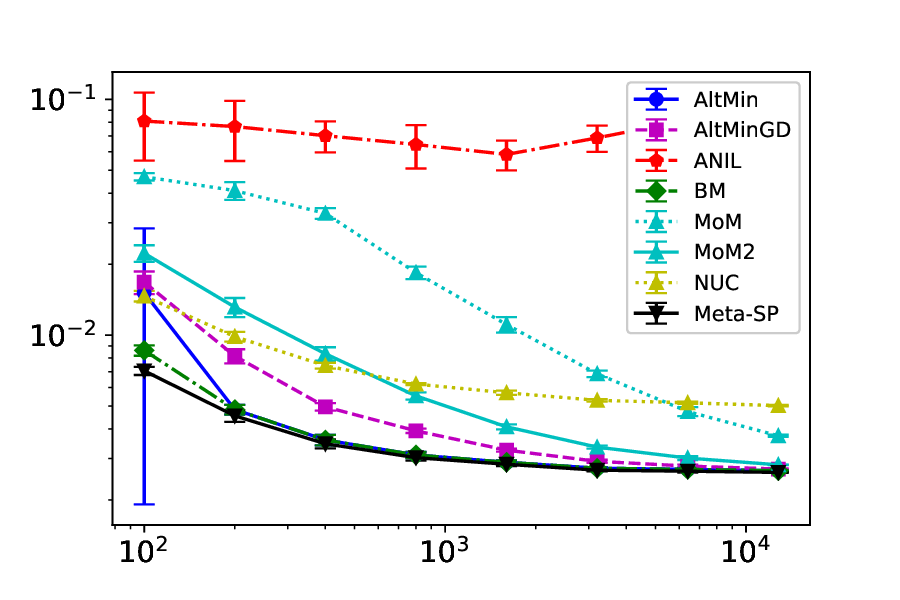}
    \includegraphics[width=.45\textwidth]{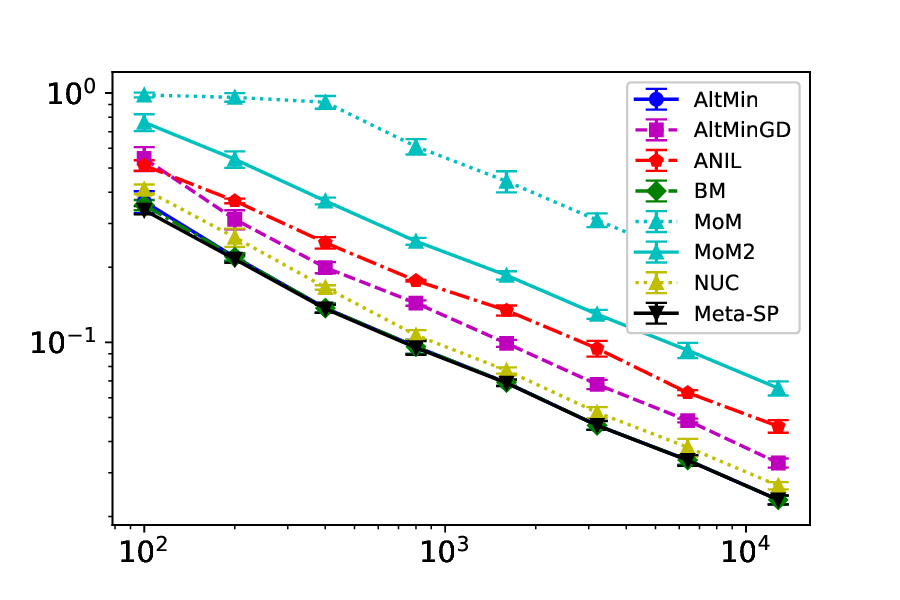}
    \caption{Evolution of $\mathbf{Dist}_1$ (left) and $\mathbf{Dist}_2$ (right) with the number of tasks $T$ for $s=5$, $m=25$ and $\sigma=1$.}
    \label{fig_T_m=25}
\end{figure}

\begin{figure}[ht]
    \centering
    \includegraphics[width=.45\textwidth]{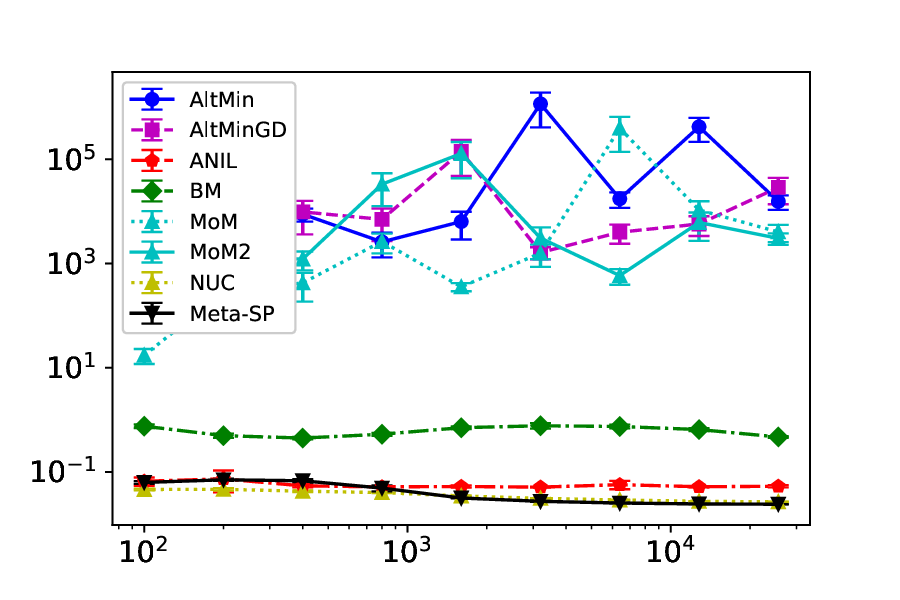}
    \includegraphics[width=.45\textwidth]{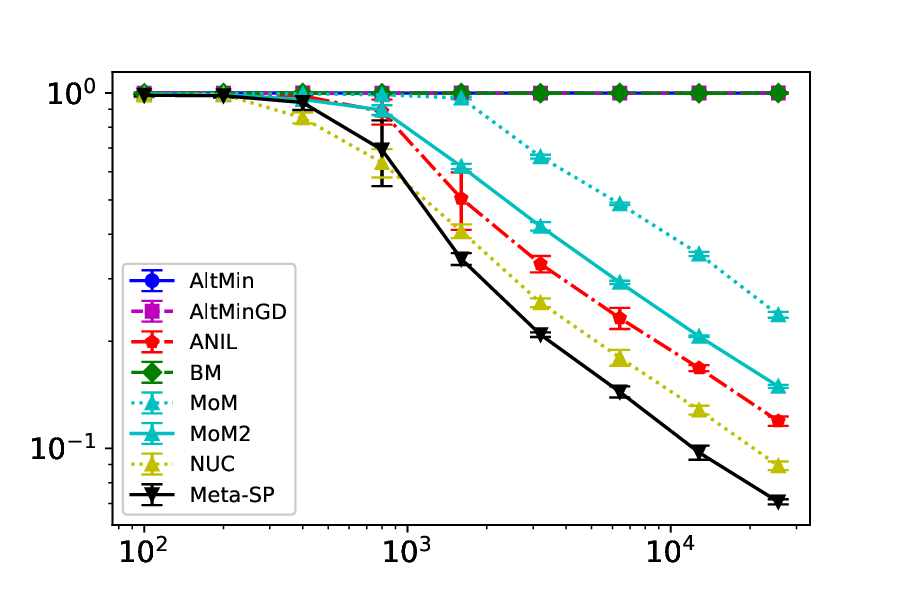}
    \caption{Evolution of $\mathbf{Dist}_1$(left) and $\mathbf{Dist}_2$(right) with the number of tasks $T$ for $s=5$, $m=5$ and $\sigma=1$.}
    \label{fig_T_m=5}
\end{figure}

Figures \ref{fig_T_m=25} and \ref{fig_T_m=5} showcase how $\mathbf{Dist}_1$ and $\mathbf{Dist}_2$ change by varying the number of tasks $T$. When $m = 25$, it is notable that \our{}, \texttt{AltMin}, and \texttt{BM} achieve lower errors compared to other methods. Notably, \texttt{ANIL}, designed to model-agnostically identify the meta-representation, consistently underperformed by other methods under the squared distances induced by the matrix Frobenius norm. When $m = s = 5$, the scenario corresponding to the extreme data-scarcity setting, methods like \texttt{AltMin}, \texttt{AltMinGD}, and \texttt{BM} fail to learn either the regression coefficients (the left panel) or the task-invariant representation (the right panel). In this scenario, our proposed approach \our{} exhibits superior performance, over other methods that still exhibit shrinking estimation error as $T$ grows, including \texttt{MoM}, \texttt{MoM2}, \ANIL{}, and \texttt{NUC}. As shown in Figures \ref{fig_T_m=25} and \ref{fig_T_m=5}, most methods exhibit stable performance across multiple runs. For the subsequent figures, we report only the average results over multiple experiments for clarity.

\begin{figure}[ht]
    \centering
    \includegraphics[width=.45\textwidth]{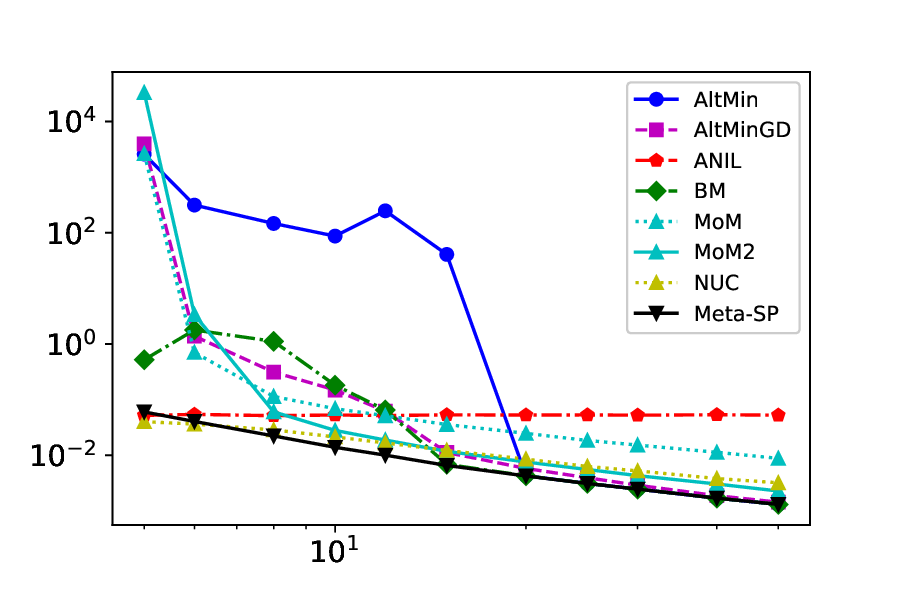}
    \includegraphics[width=.45\textwidth]{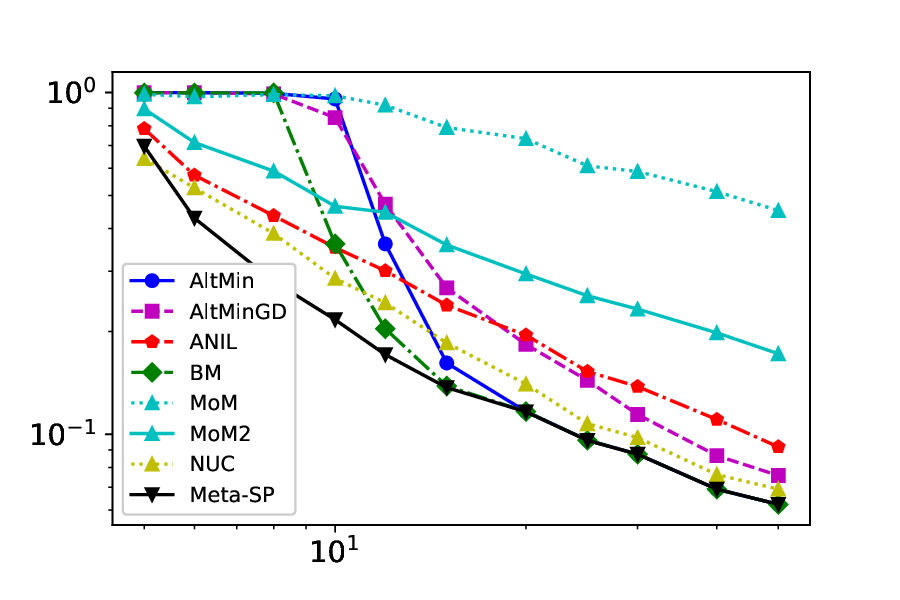}
    \caption{Evolution of $\mathbf{Dist}_1$(left) and $\mathbf{Dist}_2$(right) with the sample size $m$ for $s=5$, $T=800$ and $\sigma=1$.}
    \label{fig_m_T=800}
\end{figure}

Figure \ref{fig_m_T=800} illustrates how the two metrics change as $m$ varies. Notably, our proposed method \our{} consistently maintains strong performance across different $m$'s. Particularly, for small $m$ values ($m < 20$), \texttt{AltMin}, \texttt{AltMinGD}, and \texttt{BM} exhibit diminished performance. However, as $m$ increases, the performance of \texttt{AltMin} and \texttt{BM} becomes better, and is eventually comparable to \our{}.

\begin{figure}[ht]
    \centering
    \includegraphics[width=.45\textwidth]{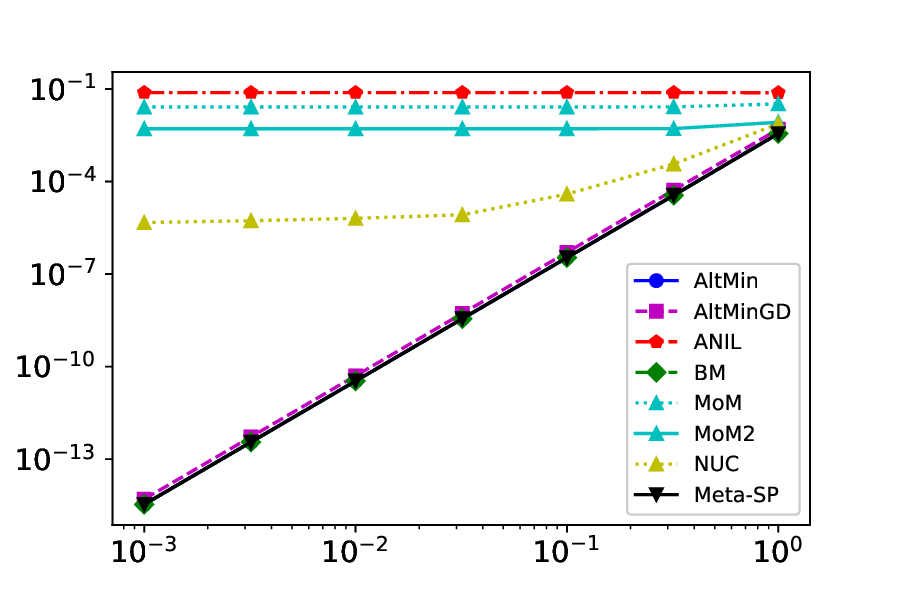}
    \includegraphics[width=.45\textwidth]{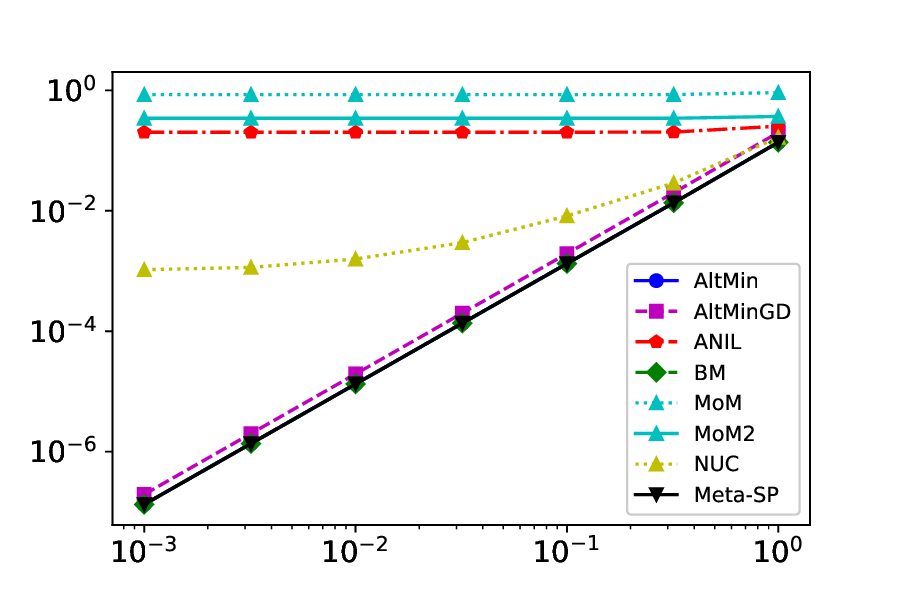}
    \caption{Evolution of $\mathbf{Dist}_1$(left) and $\mathbf{Dist}_2$(right) with variance of noise $\sigma$ for $s=5$, $T=400$ and $m=25$.}
    \label{fig_sigma_m=25}
\end{figure}

Figure \ref{fig_sigma_m=25} illustrates how the two metrics change as $\sigma$ changes. With $m = 25$, \our{}, \texttt{AltMin}, and \texttt{BM} perform better than other methods in general. Additionally, these three methods, along with \texttt{AltMinGD}, exhibit the ability to reduce errors as $\sigma$ diminishes. More comprehensive numerical results regarding situations where $m$ is very small or $\sigma=0$ can be found in the Appendix \ref{app-exp-more}. Notably, \texttt{NUC} fails to achieve zero error in the noiseless setting, since the implemented algorithm provided in \cite{boursier2022trace} did not solve the optimization problem well\footnote{As explained in the Introduction, BFGS method is directly used to solve the optimization problem, which can lead to sub-optimal numerical performance. Theoretically, supported by the results of \cite{liu2023improved}, \texttt{NUC} should achieve zero error in the noiseless setting if the numerical optimization step achieves negligible iteration errors.}.

\begin{table}
    \centering
    \begin{tabular}{|l|l|l|l|l|l|l|}
    \hline
    m=  & \our{}(Ours) & \texttt{AltMin} & \texttt{AltMinGD} & \texttt{BM}  & \texttt{NUC} & \texttt{ANIL}  \\ \hline
    5   & 12000 & /      & /        & /       & 19000    & 35000 \\ \hline
    6   & 8000  & /      & /        & /       & 12000    & 27000 \\ \hline
    8   & 4500  & /      & 58000    & 7400    & 7200     & 15000 \\ \hline
    10  & 2800  & 5100   & 17500    & 3400    & 4500     & 10000 \\ \hline
    25  & 750   & 760    & 1600     & 760     & 950      & 2600  \\ \hline
    50  & 320   & 320    & 450      & 320     & 400      & 900   \\ \hline
    75  & 220   & 220    & 270      & 220     & 240      & 480   \\ \hline
    100 & 160   & 160    & 200      & 160     & 180      & 320   \\ \hline
    \end{tabular}   
    \caption{The minimum amount of data required to achieve a sine angle distance $\le 0.1$.}
    \label{tab_mT}
\end{table}

\begin{figure}[ht]
    \centering
    \includegraphics[width=.9\textwidth]{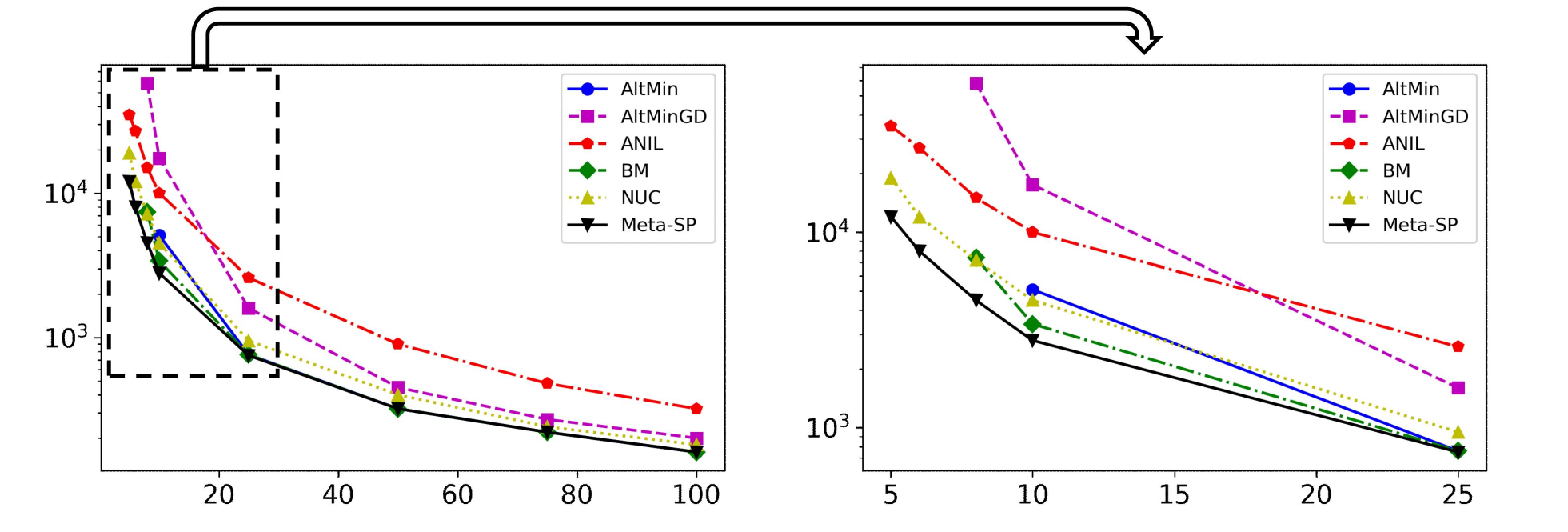}
    \caption{The empirical minimum amount of $(m, T)$ required for the sine angle distance between the estimated and the true task-invariant subspaces to be $\leq 0.1$. The horizontal axis represents the value of $m$, while the vertical axis represents the value of $T$.}
    \label{fig_mT}
\end{figure}

Figure \ref{fig_mT} along with Table \ref{tab_mT} present the empirical minimum amount of data for each method to attain a sine angle distance $\le 0.1$ over one simulation. Notably, \texttt{MoM} and \texttt{MoM2} requires an a large amount of data and are thus excluded from this representation. Strikingly, \our{} exhibits the most favorable performance in terms of data requirement, particularly when confronted with small sample size $m$ per task.

\begin{figure}[ht]
    \centering
    \subfigure[Iterations]{
    \label{fig_T=400_m=25_1}
    \includegraphics[width=.45\textwidth]{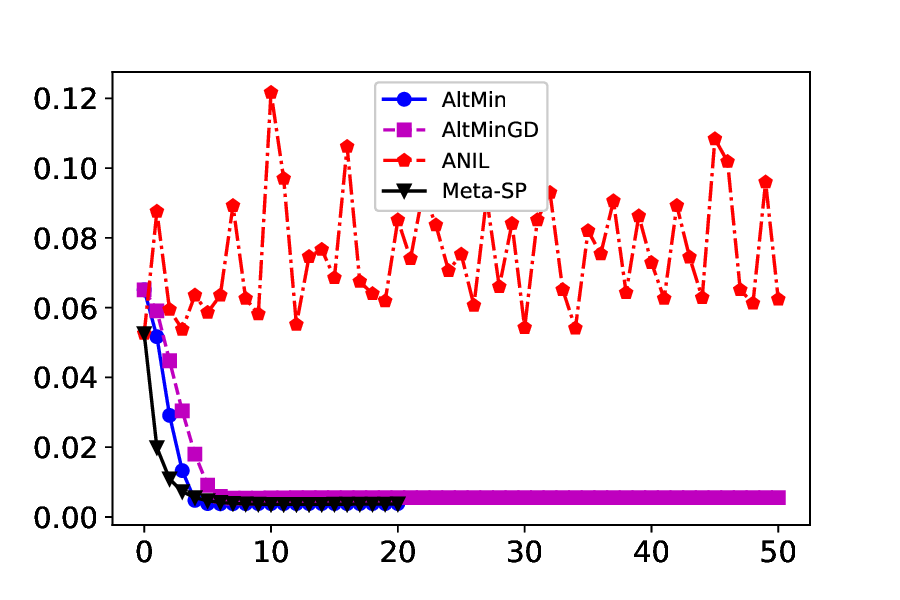}
    \includegraphics[width=.45\textwidth]{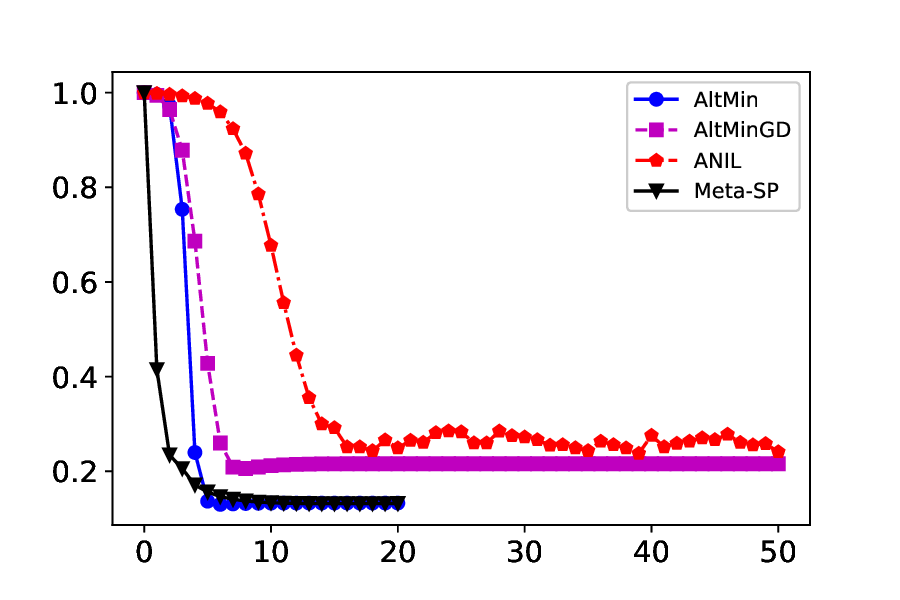}
    }
    \subfigure[Time]{
    \label{fig_T=400_m=25_2}
    \includegraphics[width=.45\textwidth]{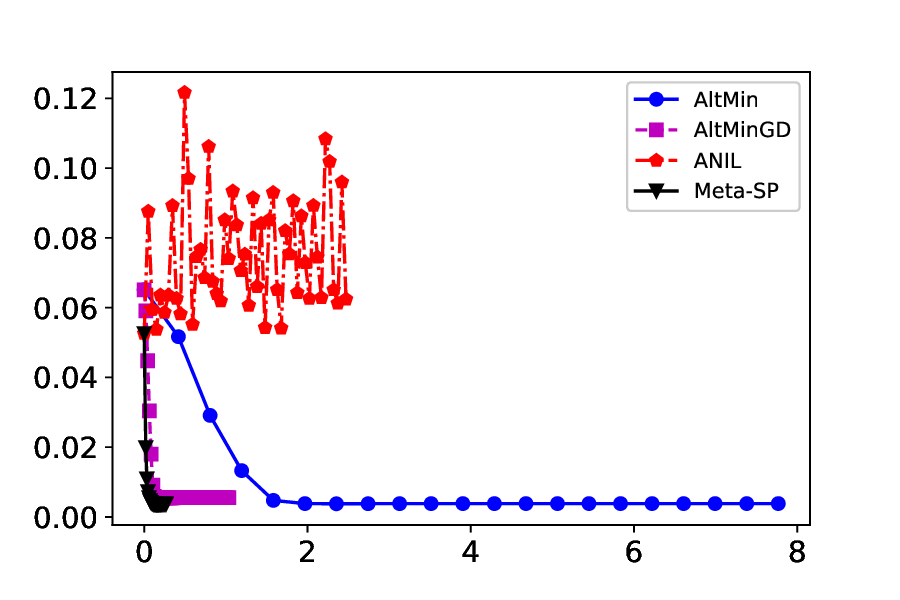}
    \includegraphics[width=.45\textwidth]{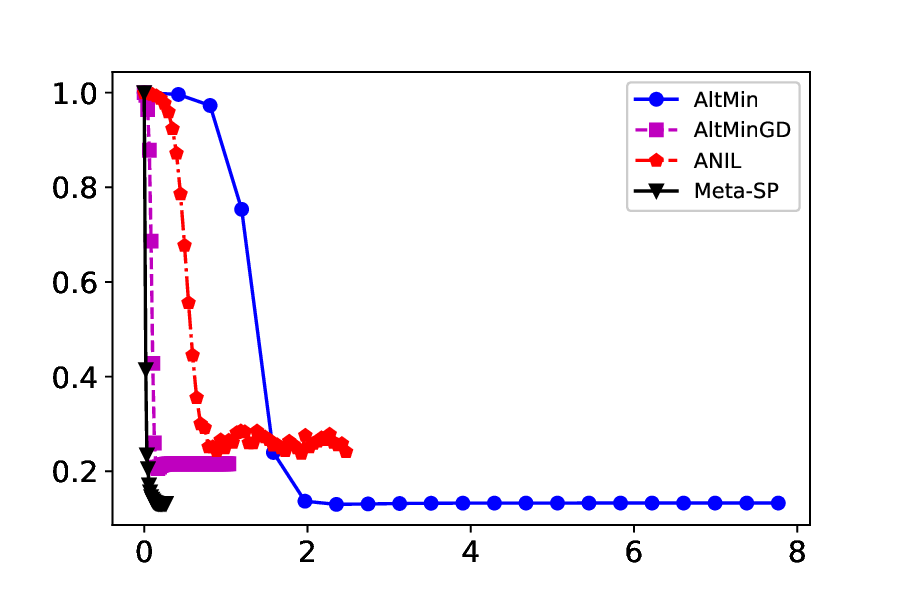}
    }
    \caption{Evolution of $\mathbf{Dist}_1$(left) and $\mathbf{Dist}_2$(right) with iterations(\ref{fig_T=400_m=25_1}) and time(\ref{fig_T=400_m=25_2}, unit: second) for $s=5$, $m=25$, $T=400$ and $\sigma=1$ in one example}
    \label{fig_T=400_m=25}
\end{figure}

Figure \ref{fig_T=400_m=25} shows how $\mathbf{Dist}_1$ and $\mathbf{Dist}_2$ change over iterations and time for the four iterative methods: \our{}, \texttt{AltMin}, \texttt{AltMinGD} and \texttt{ANIL}. Evidently, \our{} demonstrates significantly better computational efficiency compared to the other methods. Although \texttt{AltMin} requires fewer iterations, its efficiency is hampered by the larger computation time per iteration. Notably, there are instances, as demonstrated in Appendix \ref{app-exp-more}, where \texttt{AltMin} still requires more iterations than \our{} when dealing with small values of $m$.

In general, \our{} exhibits the highest computational efficiency and consistently achieves superior statistical accuracy across various scenarios. While \texttt{AltMin} and \texttt{BM} are shown to be effective when data is not excessively scarce, they are ill-suited when the sample size of each task is small. Furthermore, \texttt{AltMin} demands substantial computational resources per iteration. Its variant, \texttt{AltMinGD}, offers improved speed but lags behind other methods in terms of performance based on the two evaluation metrics. \texttt{NUC}, \texttt{ANIL}, \texttt{MoM}, and \texttt{MoM2} can effectively handle scenarios with highly restricted data. Although \texttt{NUC} performs comparably to our method (\our{}), it falls short when confronted with extremely low noise intensity. On the other hand, \texttt{ANIL} only provides a subspace estimation, yielding less satisfactory results. Despite the conciseness of \texttt{MoM} and \texttt{MoM2}, their demanding data requirements and suboptimal performance undermine their utility.

\subsection{Real Data Analysis}
\label{subsec: real}

In this section, we further examine the performance of \our{} and other competing methods in a real dataset. In particular, we use the air quality data from weather stations across different regions of China in 2023, which is originally from the China national urban air quality real-time publishing platform of the China Environmental Monitoring Station\footnote{The website of China national urban air quality real-time publishing platform is \href{https://air.cnemc.cn:18007/}{https://air.cnemc.cn:18007/} and the dataset can be downloaded from \href{https://quotsoft.net/air/\#archive}{https://quotsoft.net/air/\#archive}.}. In this example, we choose to use a working linear model to approximate the true regression function, but we acknowledge that a nonlinear regression may be more appropriate in this context and thus our results are only suggestive rather than confirmatory. However, extending our method to nonlinear models is beyond the scope of this paper. Moreover, as the weather stations in the dataset are all located in China, the coefficients of the linear models across different stations are expected to be similar. Therefore, the shared representation seems to be a reasonable assumption to leverage. Model \eqref{form0} is adopted to study the relationship between the average PM2.5 in a day and the hourly \ce{CO} value, hourly \ce{NO2} value, date, and geographical location. Specifically, the regression model for each weather station is treated as a single task, where the average value of 24-hour PM2.5 is the response variable (i.e., $y$), and the hourly \ce{CO} value (24-dimensional), hourly \ce{NO2} value (24-dimensional), date (1-dimensional), and (geographical) location coordinates (2-dimensional) are combined into covariates or features (i.e., $\tilde{\bx}$), together with the intercept (all 1 vector). After removing corrupted data, outliers, and abnormal tasks, we are left with $T = 1,210$ tasks, each with a maximum sample size of $293$ and a minimum sample size of $87$ (i.e. $m_{t} \in [87, 293]$ for $t = 1, \cdots, T$). In order to better align with the modeling assumptions and the algorithmic implementation in this paper, we preprocess the data in the following steps:
\begin{enumerate}
    \item Normalize the location coordinates of different tasks within the range $[-1,1]$.
    \item Sequentially assign date values $x$ from ${1,2,\cdots,365}$ for each task. Since PM2.5 is generally low in summer and high in winter, we transform the date values to $x=\left|x-183\right|$, and then standardize them ($x=\frac{x-\Bar{x}}{\sigma_{x}}$, where $\Bar{x}$ and $\sigma_{x}$ are the mean and variance computed from the data).
    \item Take the logarithm of the \ce{CO} and \ce{NO2} values and standardize each dimension in every task.
    \item Take the logarithm of PM2.5 values.
\end{enumerate}
We assume that the coefficients of all tasks lie in an $s$-dimensional space. 80\% of the tasks are used to train the model to get the low-dimensional subspace ($\bB$), referred to as \textit{meta tasks}. The remaining tasks are used to verify the validity of the resulting subspace, referred to as \textit{test tasks}. For each \textit{meta task} and \textit{test task}, we divide the data into training points (randomly selecting $m$ points) and test points (the rest). Training points in \textit{meta tasks} are trained with different methods to obtain the subspace $\bB$ and cross-task regression coefficients $\bTheta$. Test points in \textit{meta tasks} are used to evaluate the results of $\bTheta$. Given $\bB$, training points $(\bX_j,\by_j)$ in a \textit{test task} are used to obtain the intrinsic low-dimensional coefficients by computing
\begin{equation*}
    \bw_j=((\bX_j\bB)^{\top}(\bX_j\bB))^{-1}((\bX_j\bB)^{\top}\by_j)
\end{equation*}
and
\begin{equation*}
    \btheta_j = \bB \bw_j.
\end{equation*}
Then test points in \textit{test tasks} are used to verify the performance of the obtained coefficients, which also reflects the performance of $\bB$. The training and testing process is divided into four stages: use training points in \textit{meta tasks} to train the model with different methods(\textit{meta-train}); use test points in \textit{meta tasks} to test the model(\textit{meta-test}); and use training points in \textit{test tasks} to solve the regression problem with $\bB$(\textit{test-train}); use test points in \textit{test tasks} to test the performance(\textit{test-test}).

We focus on the results of the stage \textit{meta-test}, \textit{test-train} and \textit{test-test}. For each stage, we compute the mean relative error (MRE) of the predicted PM2.5 values for each task, and use the mean of MRE (M-MRE) over all corresponding tasks as the performance metric. Additionally, we use two other methods as baselines in \textit{test tasks}:
\begin{itemize}
    \item \texttt{Random-B}: $\bB$ is randomly generated when solving regression problems in test tasks;
    \item \texttt{Lsq-Pinv}: Combine the least square method and pseudo inverse, i.e. $\btheta_j=(\bX_j^{\top}\bX_j)^{\dag}\bX_j^{\top}\by_j$.
\end{itemize}

We set $s=6$ and the results of different methods\footnote{\texttt{ANIL} failed in the experiments of real data, so it is not shown in the results. Other approaches mentioned in Section \ref{subsec:simulated} are all compared here.} for various $m$ are shown in Tables \ref{m_mt}, \ref{m_ttr}, \ref{m_tt}. The best-performing results are highlighted in bold font and the second-best results are highlighted in italic font.

\begin{table}[H]
\centering
\begin{tabular}{|c|c|c|c|c|c|c|c|}
\hline
\textit{meta-test} & \texttt{AltMin} & \texttt{AltMinGD} & \texttt{BM} & \texttt{MoM} & \texttt{MoM2} & \texttt{NUC} & \our{} \\ \hline
$m=30$ & 0.4311 & 0.4286 & \textcolor{black}{\textit{0.4175}}  & 1.4128 & 0.7737 & 0.7206 & \textcolor{black}{\textbf{0.4044}} \\ \hline
$m=24$ & 0.4588 & 0.4429 & \textcolor{black}{\textit{0.4318}}  & 17.9870 & 0.6858 & 0.6545 & \textcolor{black}{\textbf{0.4181}} \\ \hline
$m=18$ & 0.5120 & 0.4827 & \textcolor{black}{\textit{0.4625}} & 236759.4 & 1.3979 & 0.6504 & \textcolor{black}{\textbf{0.4469}} \\ \hline
$m=12$ & 6.9935 & 0.6289 & \textcolor{black}{\textbf{0.5040}} & 5.03E+10 & 21.1079 & 0.6649 & \textcolor{black}{\textit{0.5179}} \\ \hline
\end{tabular}
\caption{M-MRE results of different methods for various $m$ values when $s=6$ in \textit{meta-test} stage}
\label{m_mt}
\end{table}

\begin{table}[H]
\centering
\begin{tabular}{|c|c|c|c|c|c|c|c|c|}
\hline
\textit{test-train} & \texttt{AltMin} & \texttt{AltMinGD} & \texttt{BM} & \texttt{MoM} & \texttt{MoM2} & \texttt{NUC} & \our{}  & \texttt{Random-B} \\ \hline
$m=30$ & \textcolor{black}{\textit{0.3197}} & 0.3218 & 0.3357 & 0.8862 & 0.5199 & 0.3351 & \textcolor{black}{\textbf{0.3163}} & 0.8975 \\ \hline
$m=24$ & \textcolor{black}{\textit{0.3109}} & 0.3112 & 0.3236 & 0.8981 & 0.4677 & 0.3202 & \textcolor{black}{\textbf{0.3044}} & 1.1655 \\ \hline
$m=18$ & \textcolor{black}{\textit{0.2834}} & 0.2848 & 0.3034 & 0.9063 & 0.5705 & 0.2952 & \textcolor{black}{\textbf{0.2834}} & 1.0572 \\ \hline
$m=12$ & 0.2456 & \textcolor{black}{\textit{0.2451}} & 0.2627 & 0.9710 & 0.4947 & 0.2493 & \textcolor{black}{\textbf{0.2439}} & 1.0220 \\ \hline
\end{tabular}
\caption{M-MRE results of different methods for various $m$ values when $s=6$ in \textit{test-train} stage}
\label{m_ttr}
\end{table}

\begin{table}[H]
\centering
\begin{tabular}{|c|c|c|c|c|c|c|c|c|c|}
\hline
\textit{test-test} & \texttt{AltMin} & \texttt{AltMinGD} & \texttt{BM} & \texttt{MoM} & \texttt{MoM2} & \texttt{NUC} & \our{}  & \texttt{Random-B} & \texttt{Lsq-Pinv} \\ \hline
$m=30$ & \textcolor{black}{\textit{0.4252}} & 0.4296 & 0.4385 & 1.6439 & 0.7493 & 0.4464 & \textcolor{black}{\textbf{0.4242}} & 3.6900 & 0.8469 \\ \hline
$m=24$ & \textcolor{black}{\textit{0.4452}} & 0.4491 & 0.4651 & 30.4750  & 0.7348 & 0.4574 & \textcolor{black}{\textbf{0.4413}} & 111.0599 & 0.6932 \\ \hline
$m=18$ & 0.4817 & \textcolor{black}{\textit{0.4783}} & 0.5217 & 5011.062 & 1.4058 & 0.4971 & \textcolor{black}{\textbf{0.4708}} & 37555.84 & 0.6681 \\ \hline
$m=12$ & \textcolor{black}{\textit{0.6171}} & 0.7382 & 1.2946 & 3.35E+12 & 118.889 & 0.6732  & \textcolor{black}{\textbf{0.6160}} & 977322.9 & 0.9575 \\ \hline
\end{tabular}
\caption{M-MRE results of different methods for various $m$ values when $s=6$ in \textit{test-test} stage}
\label{m_tt}
\end{table}

We also investigate the performance of the methods for different $r$ values. The results are shown in Tables \ref{r_mt}, \ref{r_ttr}, and \ref{r_tt} for the \textit{meta-test}, \textit{test-train}, and \textit{test-test} stages, respectively. In these tables, $m$ is set to 30, and the best results are highlighted in bold font, while the second-best results are in italic font.

\begin{table}[H]
\centering
\begin{tabular}{|c|c|c|c|c|c|c|c|}
\hline
\textit{meta-test} & \texttt{AltMin} & \texttt{AltMinGD} & \texttt{BM} & \texttt{MoM} & \texttt{MoM2} & \texttt{NUC} & \our{} \\ \hline
$s=6$ & 0.4311 & 0.4286 & \textcolor{black}{\textit{0.4175}} & 1.4128 & 0.7737 & 0.7206 & \textcolor{black}{\textbf{0.4044}} \\ \hline
$s=9$ & 0.5434 & 0.4615 & \textcolor{black}{\textit{0.4368}} & 0.4522 & 0.7257 & 0.7231 & \textcolor{black}{\textbf{0.4249}} \\ \hline
$s=12$ & 0.5480 & 0.5076 & \textcolor{black}{\textit{0.4586}} & 0.4924 & 0.7793 & 0.7260 & \textcolor{black}{\textbf{0.4515}} \\ \hline
$s=15$ & 0.6498 & 0.5889 & \textcolor{black}{\textit{0.5312}} & 0.5685 & 1.1971 & 0.7178 & \textcolor{black}{\textbf{0.4717}} \\ \hline
\end{tabular}
\caption{M-MRE results of different methods for various $s$ values when $m=30$ in \textit{meta-test} stage}
\label{r_mt}
\end{table}

\begin{table}[H]
\centering
\begin{tabular}{|c|c|c|c|c|c|c|c|c|}
\hline
\textit{test-train} & \texttt{AltMin} & \texttt{AltMinGD} & \texttt{BM} & \texttt{MoM} & \texttt{MoM2} & \texttt{NUC} & \our{}  & \texttt{Random-B} \\ \hline
$s=6$ & \textcolor{black}{\textit{0.3197}} & 0.3218 & 0.3357 & 0.8862 & 0.5199 & 0.3351 & \textcolor{black}{\textbf{0.3163}} & 0.8975 \\ \hline
$s=9$ & \textcolor{black}{\textit{0.2915}} & 0.2983 & 0.3054 & 0.2926 & 0.4224 & 0.3104 & \textcolor{black}{\textbf{0.2913}} & 1.1869 \\ \hline
$s=12$ & 0.2718 & 0.2687 & 0.2819 & \textcolor{black}{\textbf{0.2620}} & 0.3655 & 0.2837 & \textcolor{black}{\textit{0.2645}} & 1.0976 \\ \hline
$s=15$ & 0.2404 & 0.2402 & 0.2578 & \textcolor{black}{\textit{0.2364}} & 0.3253 & 0.2571 & \textcolor{black}{\textbf{0.2361}} & 0.6470 \\ \hline
\end{tabular}
\caption{M-MRE results of different methods for various $s$ values when $m=30$ in \textit{test-train} stage}
\label{r_ttr}
\end{table}

\begin{table}[H]
\centering
\begin{tabular}{|c|c|c|c|c|c|c|c|c|c|}
\hline
\textit{test-test} & \texttt{AltMin} & \texttt{AltMinGD} & \texttt{BM} & \texttt{MoM} & \texttt{MoM2} & \texttt{NUC} & \our{}  & \texttt{Random-B} & \texttt{Lsq-Pinv} \\ \hline
$s=6$ & \textcolor{black}{\textit{0.4252}} & 0.4296 & 0.4385 & 1.6439 & 0.7493 & 0.4464 & \textcolor{black}{\textbf{0.4242}} & 3.6900 & 0.8469 \\ \hline
$s=9$ & 0.4663 &\textcolor{black}{\textbf{0.4595}} & 0.4819 & 0.4652 & 0.7463 & 0.4933 & \textcolor{black}{\textit{0.4608}} & 9527.81 & 0.8469 \\ \hline
$s=12$ & 0.5130 & 0.5128 & 0.5332 & \textcolor{black}{\textit{0.5107}} & 1.0832 & 0.5418 & \textcolor{black}{\textbf{0.5085}} & 334.302 & 0.8469 \\ \hline
$s=15$ & \textcolor{black}{\textbf{0.5831}} & 0.5950 & 0.6152 & 0.5925 & 1.0421 & 0.6376 & \textcolor{black}{\textit{0.5923}} & 4.6051 & 0.8469 \\ \hline
\end{tabular}
\caption{M-MRE results of different methods for various $s$ values when $m=30$ in \textit{test-test} stage}
\label{r_tt}
\end{table}

The results in the tables show that our method \our{} consistently ranks among the top two methods in different situations and is often the best one. Additionally, the methods based on Model \eqref{form0} significantly outperform the baseline methods, indicating that Model \eqref{form0} could be appropriate for investigating the given dataset.

\section{Concluding Remarks} 
\label{sec:conclusion}

In this study, we introduced the concept of utilizing matrix rank minimization techniques to address challenges posed by multi-task linear regression problems with limited data availability. Our proposed approach, termed \texttt{Meta Subspace Pursuit} (\our{}), demonstrates efficacy in resolving such problems. Moreover, we conducted a theoretical analysis of the error bounds associated with the estimations produced by this algorithm. Through experimental validation, \our{} showcased superior accuracy and efficiency compared to existing methodologies, particularly in scenarios with exceptionally scarce data. This underscores the potential for improved solutions in cases involving very limited samples.

Nevertheless, it is important to acknowledge that, based on our experimental findings, the error bounds we established may not be the most tight. Notably, the error associated with the sine angle distance displays a gradual reduction with increasing values of $T$. This intriguing observation warrants further exploration and validation in future research endeavors.

In this paper, we assume that the rank $s$ of the common subspace is known. The immediate next step is to show that our procedure adapts to unknown $s$. Another interesting direction is to analyze related algorithms, including \our{}, under the so-called ``proportional asymptotic limit'' \citep{donoho2016high, jin2024meta, li2024understanding} ($s \asymp d \asymp n$ and $s \leq d$), an arguably more realistic high-dimensional regime than the $d \gg n$ or $d \ll n, d \rightarrow \infty$ regimes, using techniques from statistical physics \cite{dudeja2023universality, han2023universality}. 

It is also of theoretical interest to characterize the fundamental limit of estimating the common subspace when the task diversity assumption may be nearly violated, using the local minimax rate framework recently put forth in the theoretical computer science and statistics literature \citep{valiant2017automatic}. Finally, in most modern applications, the HPS linear model is at best an approximation of the actual data generating mechanism. Developing methods for its non- or semi-parametric extension, such as HPS single index models or model-agnostic settings \cite{finn2017model, raghu2020rapid}, is thus another important research direction that deserves more attention.

%%%% Acknowledgments %%%%%%%%
\section*{Acknowledgments}
This work is supported by NSFC Grants  No.12090024 (C. Zhang, X. Zhang, L. Liu) and No.12471274 (L. Liu), Shanghai Municipal of Science and Technology Grants 21JC1402900 (L. Liu), 2021SHZDZX0102 (X. Zhang, L. Liu), and Science and Technology Talent and Platform Program of Yunnan Province Grant No.202605AF35007 (L. Liu). The authors would like to thank two anonymous referees for helpful comments.

%%%% Appendices  %%%%%%%%%%
\begin{appendices}
\appendixpage
\section{Proofs} \label{app-proof}

% \subsection{Matrix Concentration Inequality} \label{app-mci}

\subsection{Proof of Theorem \ref{ripr}} \label{app-proof-ripr}
The matrix concentration inequality used to prove Theorem \ref{ripr} is stated in the following lemma.
\begin{lemma}[Matrix Bernstein inequality \citep{tropp2015introduction}]\label{mci}
Let $\bS_1, \cdots, \bS_m \in \bbR^{d_1 \times d_2}$ be independent, centered random matrices of the same size, where $\bbE \bS_j = \bm{0}$ and $\|\bS_j\| \le L$ for every $j = 1, \ldots, m$. Let $\bZ \coloneqq \sum_{j=1}^{m} \bS_j$. Define the quantity $\nu(\bZ)$, which can be viewed as the matrix extension of the variance for scalar Bernstein inequality \citep{boucheron2003concentration}, as follows:
\begin{equation*}
    \begin{aligned}
        \nu(\bZ) & \coloneqq \max{\{\|\bbE (\bZ\bZ^{\top})\|,\|\bbE (\bZ^{\top}\bZ)\|\}} \\
        & \equiv \max{\{\|\bbE (\sum_{j=1}^{m}{\bS_j\bS_j^{\top})}\|,\|\bbE (\sum_{j=1}^{m}{\bS_j^{\top}\bS_j)}\|\}}.
    \end{aligned}
\end{equation*}
Then we have:
\begin{equation*}
    \mathbb{P} (\|\bZ\| \ge u) \le (d_1 + d_2) \cdot \exp \left\{ \frac{-u^2 / 2}{\nu (\bZ) + L u / 3} \right\} \text{ for all } u \ge 0.
\end{equation*}
\end{lemma}
The proof of Lemma \ref{mci} can be found in \cite{tropp2015introduction}. This lemma helps prove Theorem \ref{ripr}.

\begin{proof}[\textbf{Proof of Theorem \ref{ripr}}]
We analyze the concentration result for each term in Equation \eqref{rip2}. Using Lemma \ref{mci}, we set $\bS_j$ to be the matrix $\frac{1}{m}(\bB\bx_{t, j}\bx_{t, j}^{\top}\bB^{\top}-\bI_r)$. Consequently, $\bZ=\frac{1}{m}\bB\bX_t^{\top}\bX_t\bB^{\top}-\bI_r$. 

We denote $\bb_l$ as the $l$-th row vector of $\bB$ and $\|\bb_l\|_2=1$. Since $\bx_{t,j} \sim \subG(\bm{0}_{d},\I_d)$, it follows that $\bbE [(\bb_l^{\top}\bx_{t,j})^2]=1$. According to $\bbE [|z|^k]\le(2\sigma^2)^{k/2}k\Gamma(k/2)$ for any $z\sim\subG(0,\sigma^2)$, we obtain $\Var[(\bb_l^{\top}\bx_{t,j})^2]\le15$. By applying Chebyshev's Inequality, we can infer that $\mathbb{P}(1-a \le (\bb_l^{\top}\bx_{t,j})^2\le 1+a)\ge1-\frac{15}{a^2}$ for any $a>0$. Consequently, we have $\|\bB^{\top}\bx_{t,j}\|_2^2=\sum_{l=1}^{r}{(\bb_l^{\top}\bx_{t,j})^2}\le (a+1)r$ with probability at least $(1-\frac{15}{a^2})^r$ for any $a>0$. 

Subsequently, $\|\bS_j\|$ can be bounded as follows:
\begin{equation*}
    \|\bS_j\| = \|\frac{1}{m}(\bB\bx_{t,j}\bx_{t,j}^{\top}\bB^{\top}-\bI)\| \le \|\frac{1}{m}\bB\bx_{t,j}\bx_{t,j}^{\top}\bB^{\top}\| + \frac{1}{m} = \frac{1}{m}\|\bB\bx_{t,j}\|_2^2 + \frac{1}{m} \le \frac{(a+1)r+1}{m}.
\end{equation*}
This holds with high probability. Since $\bS_j$ is Hermitian, we can compute the matrix variance statistic $\nu(\bZ)$ under the assumptions stated in Lemma \ref{mci}:
\begin{equation*}
    \nu(\bZ) = \|\sum_{j=1}^{m}{\mathbb{E}\bS_j^2}\|
\end{equation*}
The variance of each term is calculated as:
\begin{equation*}
    \begin{aligned}
    \mathbb{E}\bS_j^2 & = \frac{1}{m^2}(\bbE [\bB\bx_{t,j}(\bx_{t,j}^{\top}\bB^{\top}\bB\bx_{t,j})\bx_{t,j}^{\top}\bB^{\top}]-\bI) \\
    & = \frac{1}{m^2}(\bbE [\|\bB\bx_{t,j}\|_2^2\bB\bx_{t,j}\bx_{t,j}^{\top}\bB^{\top}]-\bI) \\
    & \preccurlyeq \frac{1}{m^2}\{(a+1)r\bbE [\bB\bx_{t,j}\bx_{t,j}^{\top}\bB^{\top}]-\bI\} \\
    & \preccurlyeq \frac{(a+1)r-1}{m^2}\bI
    \end{aligned}
\end{equation*}
Thus, $\|\mathbb{E}\bS_j^2\|\le\frac{(a+1)r-1}{m^2}$. As a result, we find that:
\begin{equation*}
    \nu(\bZ) \le \sum_{j=1}^{m}\|\mathbb{E}\bS_j^2\| = \frac{(a+1)r-1}{m}
\end{equation*}
Applying Lemma \ref{mci} with the aforementioned conditions and letting $u=\delta_r^{(t)}$(assuming $u\le1$), we have
\begin{equation}\label{prip1}
    \|\frac{1}{m}\bB\bX_t^{\top}\bX_t\bB^{\top}-\bI\| \le \delta_r^{(t)} \le \sqrt{\frac{8(a+1)r-4}{3m}\log{\frac{2r}{\epsilon}}}.
\end{equation}
This holds with probability at least $(1-\epsilon)(1-\frac{15}{a^2})^{r}$ for any $a>0$. If it holds for all $t$, we can set $\delta_r(\calB)=\sqrt{\frac{2(a+4)r+8}{3m}\log{\frac{2r}{\epsilon}}}$ and sum up Equation \eqref{prip1} over $t$. This leads to the conclusion that Equation \eqref{rip2} holds with the probability at least $(1-\epsilon)^T(1-\frac{15}{a^2})^{rT}$, which in turn implies that Equation \eqref{rip} holds with this probability.
\end{proof}

\subsection{Proof of Theorem \ref{thm-theta}} \label{app-proof-thm-theta}
To start the proof, we first recall some relevant definitions and results from \citep{recht2010guaranteed} and \citep{goldfarb2011convergence}.

\begin{definition}[Orthonormal basis of a subspace]
Given a set of $r$ rank-one matrices $\varPsi = \{\psi_1,\cdots,\psi_r\}$, there exists a set of $s$ orthonormal matrices $\Gamma = \{\gamma_1, \cdots, \gamma_s\}$ in the Frobenius sense\footnote{That is, $\langle \gamma_i, \gamma_j \rangle_{\F} = 0$ if $i \ne j$ and $\|\gamma_i\|_{\F} = 1$ for all $i$, such that $\sp (\Gamma) = \sp (\varPsi)$. Here the inner product between two matrices is simply Frobenius inner product.}. We call $\Gamma$ an orthonormal basis for the space $\sp (\varPsi)$. We use $P_{\Gamma} \bm\Theta$ to denote the projection of $\bm\Theta$ onto the space $\sp (\Gamma)$. By definition, $P_{\Gamma} \bm\Theta \equiv P_{\varPsi} \bm\Theta$ and
$\text{rank} (P_{\Gamma} \bm\Theta) \le r$, $\forall \bm\Theta$.
\end{definition}

\begin{definition}[SVD basis of a matrix]
Assume that the rank-$r$ matrix $\bm\Theta_r$ has the SVD $\bm\Theta_r=\sum_{i=1}^r \sigma_iu_iv_i^{\top}$. $\Gamma:={u_1v_1^{\top},u_2v_2^{\top},\cdots,u_rv_r^{\top}}$ is
called an SVD basis for the matrix $\bm\Theta_r$. Note that the elements in $\Gamma$ are orthonormal rank-one matrices.
\end{definition}

\begin{proposition}\label{prop1}
Suppose that the linear operator $\calB:\mathbb{R}^{T \times d} \rightarrow \mathbb{R}^{N \times 1}$ satisfies the RIP with constant $\delta_r(\calB)$. Let $\varPsi$ be an arbitrary orthonormal subset of $\mathbb{R}^{T \times d}$ such that
$\text{rank}(P_{\varPsi}\bm\Theta) \le r, \forall \bm\Theta\in\mathbb{R}^{d \times T}$. Then, for all $b\in \mathbb{R}^{N \times 1}$ and $\bm\Theta\in\mathbb{R}^{T \times d}$ the following properties hold:
\begin{equation*}
\begin{aligned}
    & \|P_{\varPsi}\calB^{\ast}b\| \le \sqrt{1+\delta_r(\calB)}\|b\|_2, \\
    & (1-\delta_r(\calB))\|P_{\varPsi}\bm\Theta\|_{\F} \le \|P_{\varPsi}\calB^{\ast}\calB P_{\varPsi}\bm\Theta\|_{\F} \le (1+\delta_r(\calB))\|P_{\varPsi}\bm\Theta\|_{\F}
\end{aligned}
\end{equation*}
\end{proposition}

\begin{proposition}\label{prop2}
Suppose that the linear operator $\calB:\mathbb{R}^{T \times d} \rightarrow \mathbb{R}^{N \times 1}$ satisfies the RIP with constant $\delta_r(\calB)$. Let $\varPsi$,$\varPsi^{'}$ be an arbitrary orthonormal subset of $\mathbb{R}^{T \times d}$ such that
$\text{rank}(P_{\varPsi\cup\varPsi^{'}}\bm\Theta) \le r, \forall \bm\Theta\in\mathbb{R}^{T \times d}$. Then the following property hold:
\begin{equation*}
\begin{aligned}
    \|P_{\varPsi}\calB^{\ast}\calB(I-P_{\varPsi})\bm\Theta\|_{\F} \le \delta_r(\calB)\|(I-P_{\varPsi})\bm\Theta\|_{\F}, \forall \bm\Theta\in\sp (\varPsi^{'})
\end{aligned}
\end{equation*}
\end{proposition}

\begin{lemma}\label{plem}
Suppose $\bTheta$ is the best rank-$s$ approximation to the matrix $\hat{\bTheta}$, and $\Gamma$ is an SVD basis of $\bTheta$. Then for any rank-$s$ matrix $\bTheta_s$ and SVD basis $\Gamma_s$ of $\bTheta_s$, we have
\begin{equation*}
    \Vert P_B\bTheta-P_B\hat{\bTheta} \Vert_\F \le \Vert P_B\bTheta_s-P_B\hat{\bTheta} \Vert_\F,
\end{equation*}
where $B$ is any orthonormal set of matrices satisfying $\sp(\Gamma\cup\Gamma_s)\subseteq B$.
\end{lemma}

\begin{proof}[\textbf{Proof of Lemma \ref{plem}}]
    Since $\bTheta$ is the best rank-$s$ approximation to the matrix $\hat{\bTheta}$ and the rank of $\bTheta_s$ is $s$, $\Vert \bTheta-\hat{\bTheta} \Vert_\F \le \Vert \bTheta_s-\hat{\bTheta} \Vert_\F$. Hence,
    \begin{equation*}
        \Vert P_B(\bTheta-\hat{\bTheta}) \Vert_\F^2 + \Vert (I-P_B)(\bTheta-\hat{\bTheta}) \Vert_\F^2 \le \Vert P_B(\bTheta_r-\hat{\bTheta}) \Vert_\F^2 + \Vert (I-P_B)(\bTheta_r-\hat{\bTheta}) \Vert_\F^2.
    \end{equation*}
    Since $(I-P_B)\bTheta=0$ and $(I-P_B)\bTheta_r=0$, this reduces to the result.
\end{proof}

Provided here is a pivotal lemma that sets apart the proof in this paper from the approach presented in \cite{goldfarb2011convergence}.

\begin{lemma}\label{klem}
Given the linear operator $\calB=\mathcal{A}/\sqrt{m}$ and the sub-Gaussian noise vector $\bEps$, the following inequality holds with probability at least $(1-\frac{2}{d}-\frac{14}{m}-\frac{208}{dm})^T$:
\begin{equation*}
    \|\calB^{\ast}\bEps\|_{\op} \leq \|\calB^{\ast}\bEps\|_{\F} \le \sqrt{2dT\sigma^2}.
\end{equation*}
\end{lemma}

\begin{proof}[\textbf{Proof of Lemma \ref{klem}}]
As $\|\calB^{\ast}\bEps\|_{\F}^2=\sum_{t=1}^{T}\|\bX_t^{\top}\beps_t\|_2^2/m$, we can analyze $\|\bX_t^{\top}\beps_t\|_2^2/m$ for each task $t$ individually. The expectation of it is $\bbE [\|\bX_t^{\top}\beps_t\|_2^2/m]=\bbE [\frac{1}{m}\beps_t^{\top}\bX_t\bX_t^{\top}\beps_t]=\frac{1}{m}\bbE [\beps_t^{\top}\bbE [\bX_t\bX_t^{\top}]\beps_t]=d\sigma^2$. According to $\bbE [|z|^k]\le(2\sigma^2)^{k/2}k\Gamma(k/2)$ for any $z\sim\subG(0,\sigma^2)$, the corresponding variance is 
\begin{align*}
\Var[\|\bX_t^{\top}\beps_t\|_2^2/m] & = \bbE [(\sum_{j=1}^{m}\sum_{k=1}^{m}{\bx_{t,j}^{\top}\bx_{t,k}\varepsilon_{t,j}\varepsilon_{t,k}})^2/m^2]-(\bbE [\|\bX_t^{\top}\beps_t\|_2^2/m])^2 \\
& \le [15m(d^2+14d)+(m^2-m)(d^2+2d)]\sigma^4/m^2-d^2\sigma^4 \\
& = \frac{(2dm^2+14d^2m+208dm)\sigma^4}{m^2}.
\end{align*}
By Chebyshev’s inequality, we can establish 
\begin{align*}
\mathbb{P}(d\sigma^2-t\le\|\bX_t^{\top}\beps_t\|_2^2/m\le d\sigma^2+u)\ge1-\frac{(2dm^2+14d^2m+208dm)\sigma^4}{u^2m^2}
\end{align*}
for any $u > 0$. Here, We take $u = d\sigma^2$, resulting in $\|\bX_t^{\top}\beps_t\|_2^2/m\le 2d\sigma^2$ with a probability at least $1-\frac{2}{d}-\frac{14}{m}-\frac{208}{dm}$. By summing up these inequalities over $t$, the lemma is proved.
\end{proof}

Then we give the proof of Theorem \ref{thm-theta}.

\begin{proof}[\textbf{Proof of Theorem \ref{thm-theta}}]
Assuming that RIP holds, as established in Theorem \ref{ripr}, let $\Bar{\Gamma}^{\ast}$ and $\Bar{\Gamma}^k$ be the SVD basis of $\bm\Theta^{\ast}$ and $\bm\Theta^{(k)}$, respectively. We denote an orthonormal basis of the subspace $\sp (\Bar{\Gamma}^{\ast}\cup\Bar{\Gamma}^k)$ as $\varPsi_k$ 
Starting with the relation $\|\bm\Theta^{\ast}-\bm\Theta^{(k+1)}\|_{\F} = \|P_{\varPsi_{k+1}}\bm\Theta^{\ast}-P_{\varPsi_{k+1}}\bm\Theta^{(k+1)}\|_{\F}$ and Lemma \ref{plem} (note that one can simply conduct a verbatim application of the lemma once replacing $\bTheta_r$ by $\bTheta^*$.), we can further expand this as follows:
\begin{equation*}
\begin{aligned}
    & \|\bm\Theta^{\ast}-\bm\Theta^{(k+1)}\|_{\F} \\ 
    =& \|P_{\varPsi_{k+1}}\bm\Theta^{\ast}-P_{\varPsi_{k+1}}\bm\Theta^{(k+1)}\|_{\F} \\
    =& \|P_{\varPsi_{k+1}}\bm\Theta^{\ast}-P_{\varPsi_{k+1}}\hat{\bm\Theta}^{(k+1)}+P_{\varPsi_{k+1}}\hat{\bm\Theta}^{(k+1)}-P_{\varPsi_{k+1}}\bm\Theta^{(k+1)}\|_{\F} \\
    \le & \|P_{\varPsi_{k+1}}\bm\Theta^{\ast}-P_{\varPsi_{k+1}}\hat{\bm\Theta}^{(k+1)}\|_{\F} + \|P_{\varPsi_{k+1}}\hat{\bm\Theta}^{(k+1)}-P_{\varPsi_{k+1}}\bm\Theta^{(k+1)}\|_{\F} \\
    \le & 2\|P_{\varPsi_{k+1}}\bm\Theta^{\ast}-P_{\varPsi_{k+1}}\hat{\bm\Theta}^{(k+1)}\|_{\F}
\end{aligned}
\end{equation*}
Here, we define $Z^k=\bm\Theta^{\ast}-\bm\Theta^{(k)}$, which results in $\hat{\bm\Theta}^{(k+1)}=\bm\Theta^{(k)}+\gamma \calB^{\ast}(\calB Z^k+\frac{\bEps}{\sqrt{m}})$. Continuing the derivation, we obtain:
\begin{equation*}
\begin{aligned}
    & \|P_{\varPsi_{k+1}}\bm\Theta^{\ast}-P_{\varPsi_{k+1}}\hat{\bm\Theta}^{(k+1)}\|_{\F} \\
    = & \|P_{\varPsi_{k+1}}\bm\Theta^{\ast}-P_{\varPsi_{k+1}}\bm\Theta^{(k)}-\gamma P_{\varPsi_{k+1}}\calB^{\ast}\calB Z^k-\gamma P_{\varPsi_{k+1}}\calB^{\ast}\frac{\bEps}{\sqrt{m}}\|_{\F} \\
    = & \|P_{\varPsi_{k+1}}Z^k-\gamma P_{\varPsi_{k+1}}\calB^{\ast}\calB P_{\varPsi_{k+1}}Z^k-\gamma P_{\varPsi_{k+1}}\calB^{\ast}\calB(I-P_{\varPsi_{k+1}})Z^k-\gamma P_{\varPsi_{k+1}}\calB^{\ast}\frac{\bEps}{\sqrt{m}}\|_{\F} \\
    \le & \|(I-\gamma P_{\varPsi_{k+1}}\calB^{\ast}\calB P_{\varPsi_{k+1}})P_{\varPsi_{k+1}}Z^k\|_{\F} + \gamma \|P_{\varPsi_{k+1}}\calB^{\ast}\calB(I-P_{\varPsi_{k+1}})Z^k\|_{\F} + \frac{\gamma}{\sqrt{m}} \|P_{\varPsi_{k+1}}\calB^{\ast}\bEps\|_{\F}. 
\end{aligned}
\end{equation*}

By utilizing Proposition \ref{prop1} and \ref{prop2}, the following inequalities are established with a probability at least $(1-\epsilon)(1-\frac{15}{a^2})^{3s}(1-\frac{2}{d}-\frac{14}{m}-\frac{208}{dm})^T$ for any $a>0$:
\begin{equation*}
    \begin{aligned}
    & \|P_{\varPsi_{k+1}}\bm\Theta^{\ast}-P_{\varPsi_{k+1}}\hat{\bm\Theta}^{(k+1)}\|_{\F} \\
    \le & (1-\gamma+\gamma\delta_{2s}(\calB)) \|P_{\varPsi_{k+1}}Z^k\|_{\F} + \gamma\delta_{3s}(\calB)\|(I-P_{\varPsi_{k+1}})Z^k\|_{\F} + \frac{\gamma}{\sqrt{m}} \|\calB^{\ast}\bEps\|_{\F} \\
    \le & (1-\gamma+2\gamma\delta_{3s}(\calB))\|Z^k\|_{\F} + \sqrt{\frac{2dT\sigma^2}{m}}.
    \end{aligned}
\end{equation*}
This implies $\|Z^{k+1}\|_\F\le 2(1-\gamma+2\gamma\delta_{3s}(\calB))\|Z^k\|_{\F} + \sqrt{\frac{8dT\sigma^2}{m}}$. Consequently, with $\frac{1}{2(1-2\delta_{3s})}<\gamma<\frac{1}{1-2\delta_{3s}}$, i.e. $0<2(1-\gamma+2\gamma\delta_{3s}(\calB))<1$, we have:
\begin{equation*}
    \begin{aligned}
    & \|\bm\Theta^{\ast}-\bm\Theta^{(k)}\|_{\F} \\
    \le & \left[2(1-\gamma+2\gamma\delta_{3s}(\calB))\right]^k\|Z^0\|_{\F} + \{[2(1-\gamma+2\gamma\delta_{3s}(\calB))]+[2(1-\gamma+2\gamma\delta_{3s}(\calB))]^2+\cdots \\
    & +[2(1-\gamma+2\gamma\delta_{3s}(\calB))]^k\}\sqrt{\frac{8dT\sigma^2}{m}} \\
    \le & \left[2(1-\gamma+2\gamma\delta_{3s}(\calB))\right]^k\|Z^0\|_{\F} + \frac{2\sqrt{2}}{1-2(1-\gamma+2\gamma\delta_{3s}(\calB))}\sqrt{\frac{dT\sigma^2}{m}}.
    \end{aligned}
\end{equation*}

If $k$ is large enough, i.e. with $k\rightarrow\infty$, we have
\begin{equation*}
    \|\bm\Theta^{\ast}-\bm\Theta^{(k)}\|_{\F} \le O(\sqrt{\frac{\sigma^2dT}{m}}).
\end{equation*}
\end{proof}

\subsection{Proof of Theorem \ref{thm-B}} \label{app-proof-thm-B}
We will draw on a similar result, namely Lemma 15 in \cite{tripuraneni2021provable}, and adopt their methodologies to prove Theorem \ref{thm-B}.
\begin{proof}[\textbf{Proof of Theorem \ref{thm-B}}]
We begin the proof by defining $f (\tilde{\bW}) \coloneqq \|\tilde{\bW} \bB - \bW^{\ast} \bB^{\ast}\|_{\F}^2$ for short. Since $\tilde{\bW}^{\ast}$ minimizes $f (\tilde{\bW})$, it follows from the first-order condition $\frac{\diff f}{\diff \tilde{\bW}}=0$ that $\tilde{\bW}^{\ast}=\bW^{\ast}\bB^{\ast}\bB^{\top}$. Therefore for any $\bW$, we have $\|\bW^{\ast}\bB^{\ast}\bB^{\top}\bB-\bW^{\ast}\bB^{\ast}\|_{\F}^2\le\|\bW\bB-\bW^{\ast}\bB^{\ast}\|_{\F}^2$. Moreover,
\begin{equation*}
\begin{aligned}
    \|\bW^{\ast}\bB^{\ast}\bB^{\top}\bB-\bW^{\ast}\bB^{\ast}\|_{\F}^2 &= \|\bW^{\ast}\bB^{\ast}(\bB^{\top}\bB-\bI)\|_{\F}^2 \\
    &= \|\bW^{\ast}\bB^{\ast}\bB_{\bot}^{\top}\bB_{\bot}\|_{\F}^2 \\
    &= \|\bW^{\ast}\bB^{\ast}\bB_{\bot}^{\top}\|_{\F}^2 \\
    &= \Tr({\bW^{\ast}}^{\top}\bW^{\ast}\bB^{\ast}\bB_{\bot}^{\top}\bB_{\bot}{\bB^{\ast}}^{\top}) \\
    &\ge \bar{\lambda}_{s}({\bW^{\ast}}^{\top}\bW^{\ast})\|\bB^{\ast}\bB_{\bot}^{\top}\|_{\F}^2 \\
    &\ge \bar{\lambda}_{s}(\bW^{\ast}{\bW^{\ast}}^{\top})\|\bB^{\ast}\bB_{\bot}^{\top}\|_2^2.
\end{aligned}
\end{equation*}
In the above inequalities, $\bar{\lambda}_{s} (\mathbf{A})$ represents the $s$-th largest eigenvalue of a given matrix $\mathbf{A}$ and $\Tr(\mathbf{A})$ represents the trace of $\mathbf{A}$. It is worth noting that the second-to-last inequality follows from the property that for any positive semi-definite matrices $\mathbf{M}$ and $\mathbf{N}$, $\Tr (\mathbf{M}\mathbf{N})\ge\sigma_{min}(\mathbf{M})\Tr (\mathbf{N})$. Let $\bm\Theta^{(k)}$ and $\bB^{(k)}$ be the output obtained by \our{} after $k$ iterations. As $\bar{\lambda}_{s}(\bW^{\ast}{\bW^{\ast}}^{\top})=\lambda_sT$, it follows that:
\begin{equation*}
    \sin{\angle(\bB^{(k)},\bB^{\ast})} = \|{\bB^{\ast}}\bB_{k\bot}^{\top}\|_2 \le \frac{\|\bm\Theta^{\ast}-\bm\Theta^{(k)}\|_{\F}}{\sqrt{\lambda_s T}}\le\frac{\|\bm\Theta^{\ast}-\bm\Theta^{(k)}\|_{\F}}{\sqrt{L_{s} T}}.
\end{equation*}
Combining this result with \eqref{res1-2}, we immediately have \eqref{res2-2}.
\end{proof}

\section{More Experimental Results and Details of Simulated Data Experiments} \label{app-exp}
In this part, we show more numerical results and give the details of all the experiments in Section \ref{subsec:simulated}.
\subsection{More Experimental Results} \label{app-exp-more}
We previously demonstrated the trend of metrics' evolution with varying $T$ for cases where $m=25$ and $m=5$, depicted in Figures \ref{fig_T_m=25} and \ref{fig_T_m=5}. Now, we extend our analysis to the scenario where $m=10$, as illustrated in Figure \ref{fig_T_m=10}. Similar to our previous observations, our proposed method \our{} continues to outperform others under these conditions. Moreover, \texttt{AltMin}, \texttt{AltMinGD}, and \texttt{BM} exhibit strong performance when $T$ is relatively large. This underscores the importance of having both sufficiently large values of $m$ and $T$ to harness the effectiveness of these three methods.

\begin{figure}[htbp!]
    \centering
    \includegraphics[width=.45\textwidth]{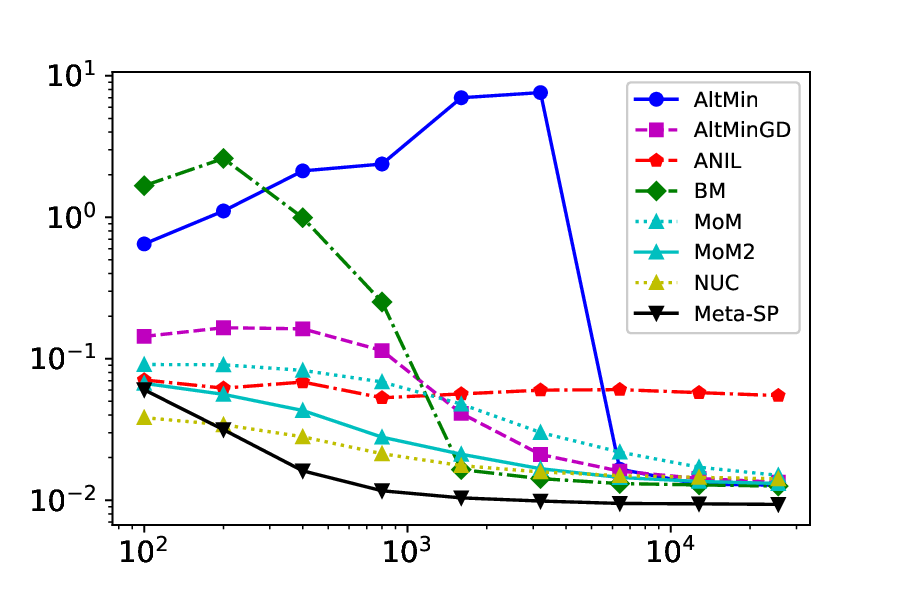}
    \includegraphics[width=.45\textwidth]{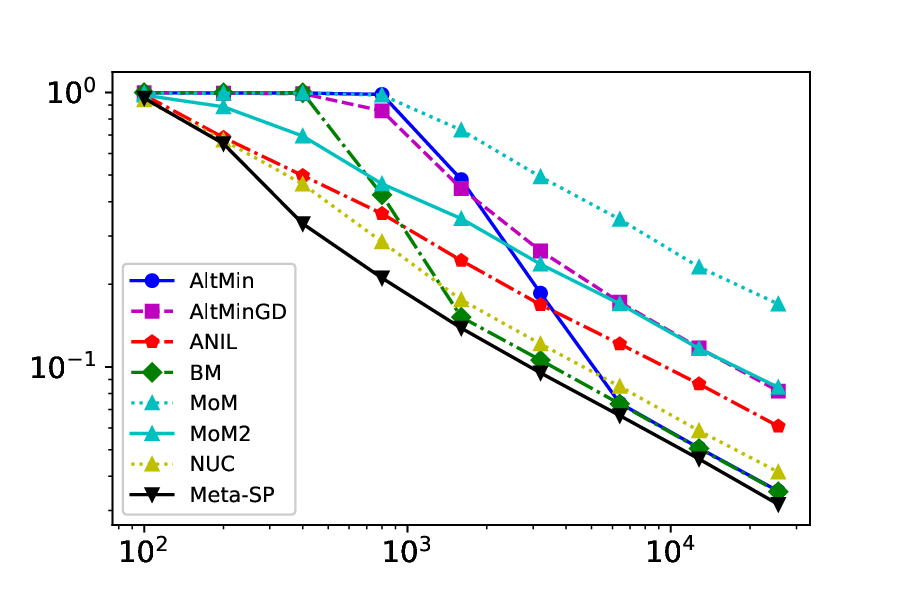}
    \caption{Evolution of $\mathbf{Dist}_1$(left) and $\mathbf{Dist}_2$(right) with the number of tasks $T$ for $s=5$, $m=10$ and $\sigma=1$.}
    \label{fig_T_m=10}
\end{figure}

\begin{figure}[htbp!]
    \centering
    \includegraphics[width=.45\textwidth]{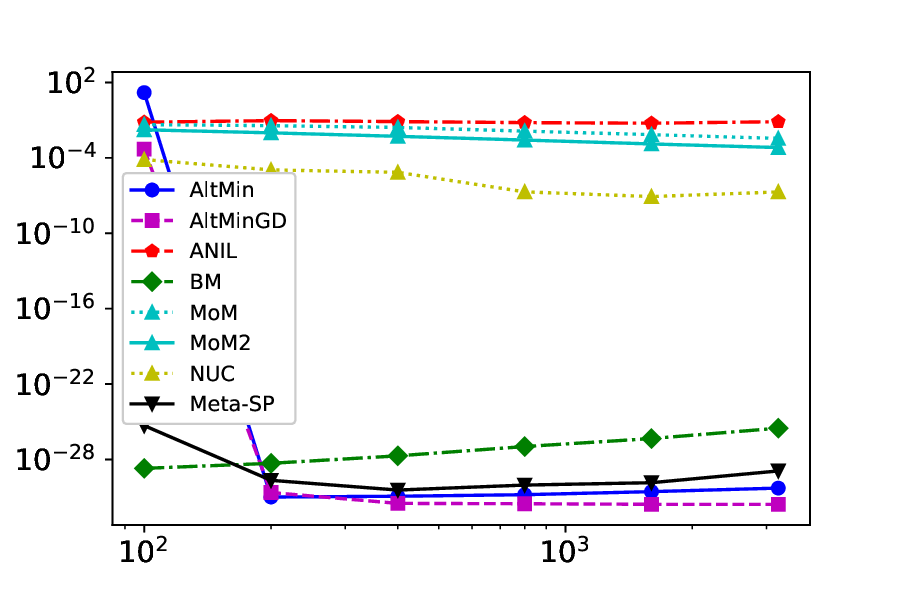}
    \includegraphics[width=.45\textwidth]{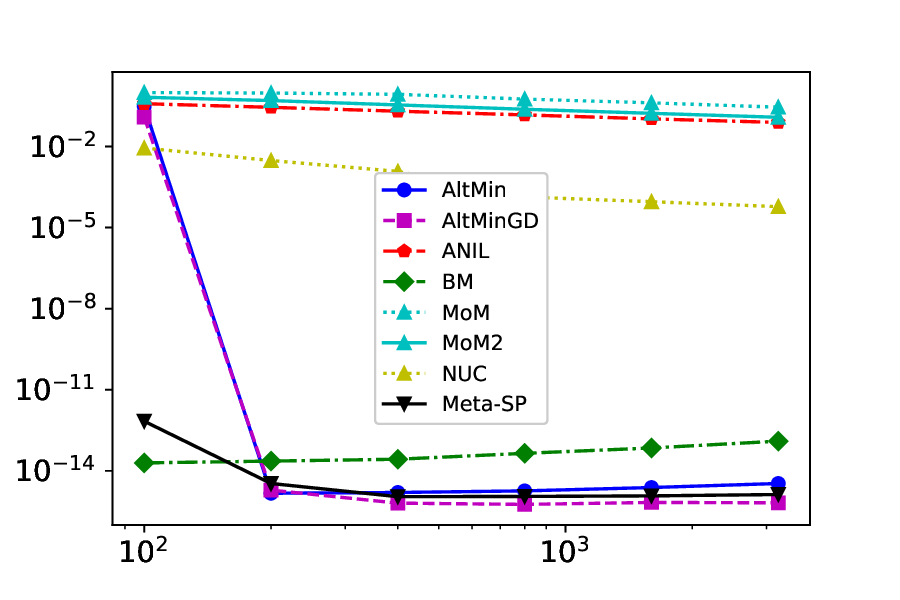}
    \caption{Evolution of $\mathbf{Dist}_1$(left) and $\mathbf{Dist}_2$(right) with the number of tasks $T$ for $s=5$, $m=25$ and $\sigma=0$.}
    \label{fig_T_m=25_sigma=0}
\end{figure}

\begin{figure}[htbp!]
    \centering
    \includegraphics[width=.45\textwidth]{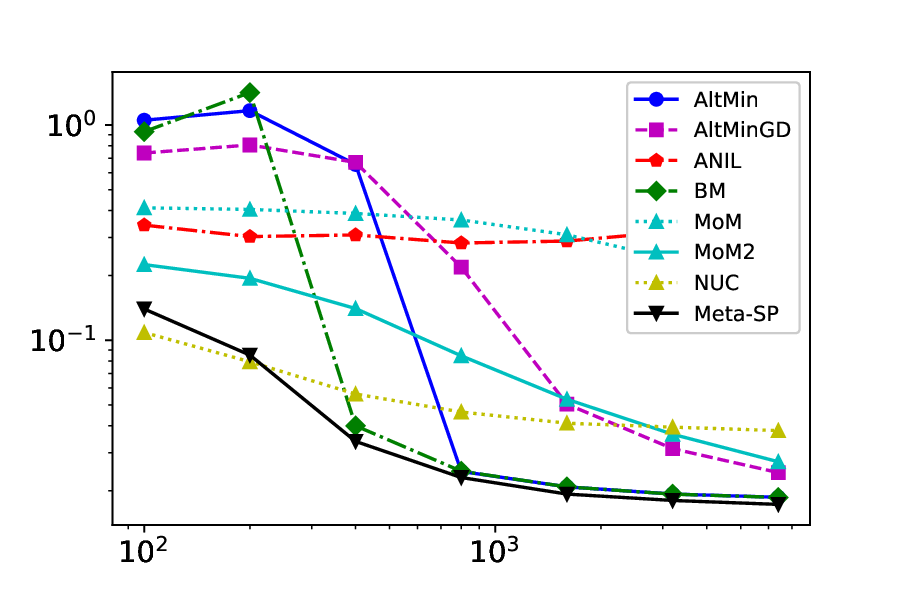}
    \includegraphics[width=.45\textwidth]{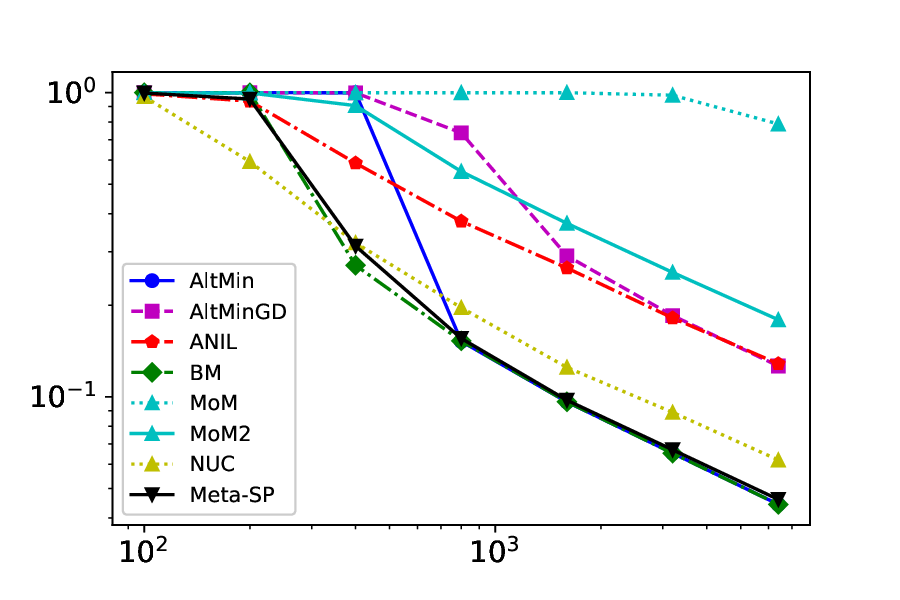}
    \caption{Evolution of $\mathbf{Dist}_1$(left) and $\mathbf{Dist}_2$(right) with the number of tasks $T$ for $s=25$, $m=40$ and $\sigma=1$.}
    \label{fig_T_m=40_s=25}
\end{figure}

In Figure \ref{fig_T_m=25_sigma=0}, we present results identical to those shown in Figure \ref{fig_T_m=25}, with the sole exception of setting $\sigma=0$. Notably, \texttt{AltMin}, \texttt{AltMinGD}, \texttt{BM}, and \our{} achieve remarkably low error rates, aligning with our expectations. In contrast, \texttt{NUC} underperforms in this scenario, while \texttt{ANIL}, \texttt{MoM}, and \texttt{MoM2} exhibit even poorer performance. These findings are in line with the outcomes depicted in Figure \ref{fig_sigma_m=25}.

We also investigated a more challenging scenario with $s=25$. Figure \ref{fig_T_m=40_s=25} and Figure \ref{fig_T_m=25_s=25} illustrate the evolution of the metrics with respect to $T$ for cases when $m=40$ and $m=25$. The comparison of results closely resembles that of the $s=5$ scenario with $m=10$ and $m=5$. These findings suggest that the methods can effectively address such challenges even as $s$ increases.

Figure \ref{fig_m_T=1600} displays the same experiment as shown in Figure \ref{fig_m_T=800}, except for $T=1600$. The results in these two scenarios are notably similar.

\begin{figure}[htbp!]
    \centering
    \includegraphics[width=.45\textwidth]{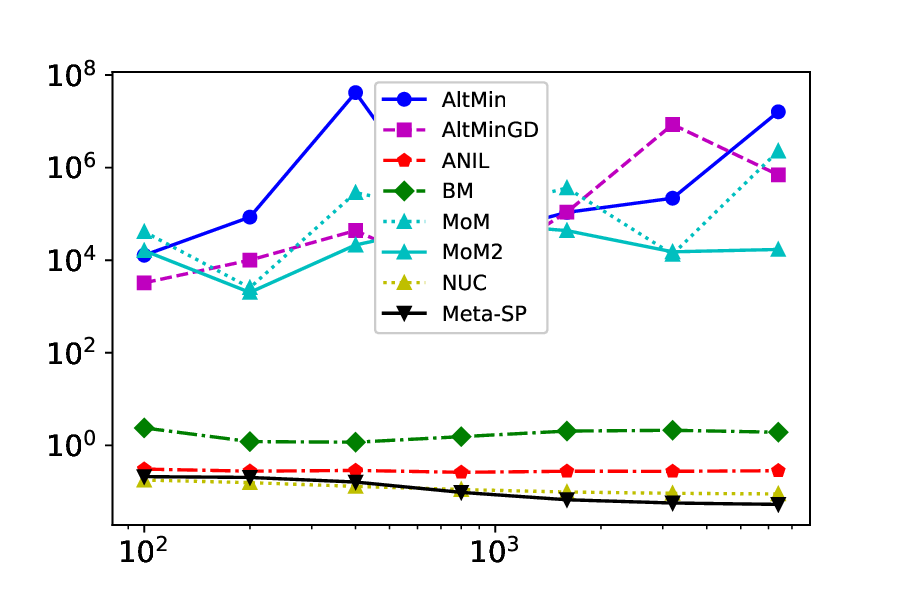}
    \includegraphics[width=.45\textwidth]{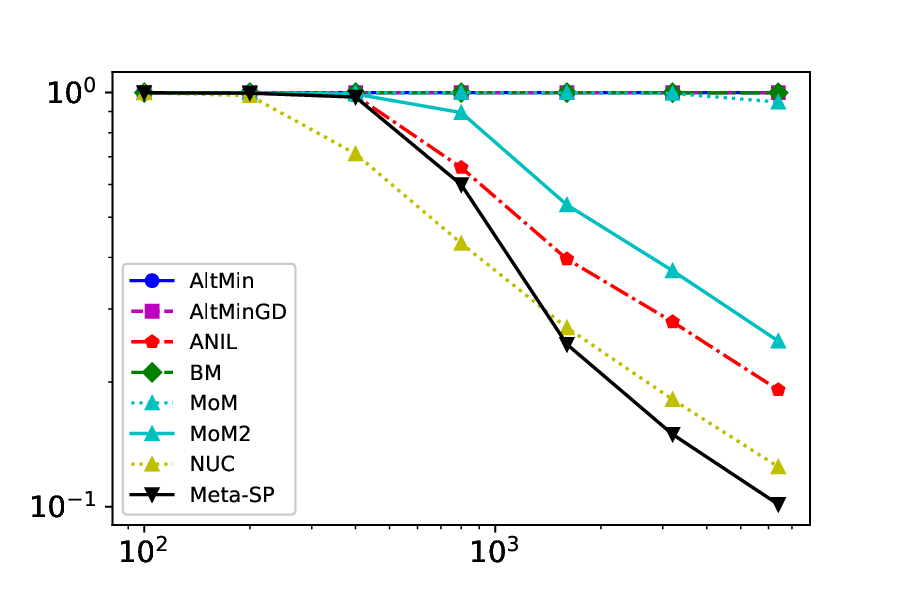}
    \caption{Evolution of $\mathbf{Dist}_1$(left) and $\mathbf{Dist}_2$(right) with the number of tasks $T$ for $s=25$, $m=25$ and $\sigma=1$}
    \label{fig_T_m=25_s=25}
\end{figure}

\begin{figure}[htbp!]
    \centering
    \includegraphics[width=.45\textwidth]{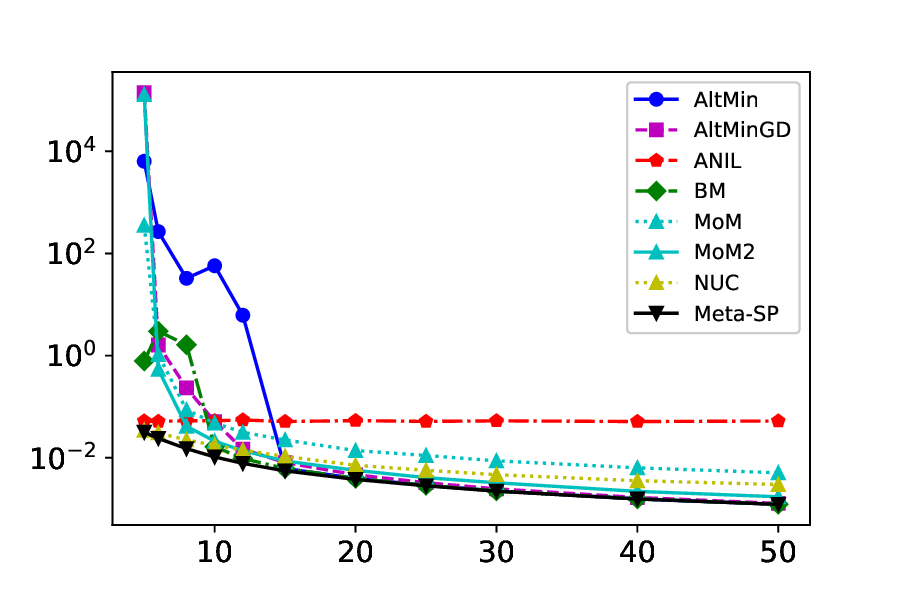}
    \includegraphics[width=.45\textwidth]{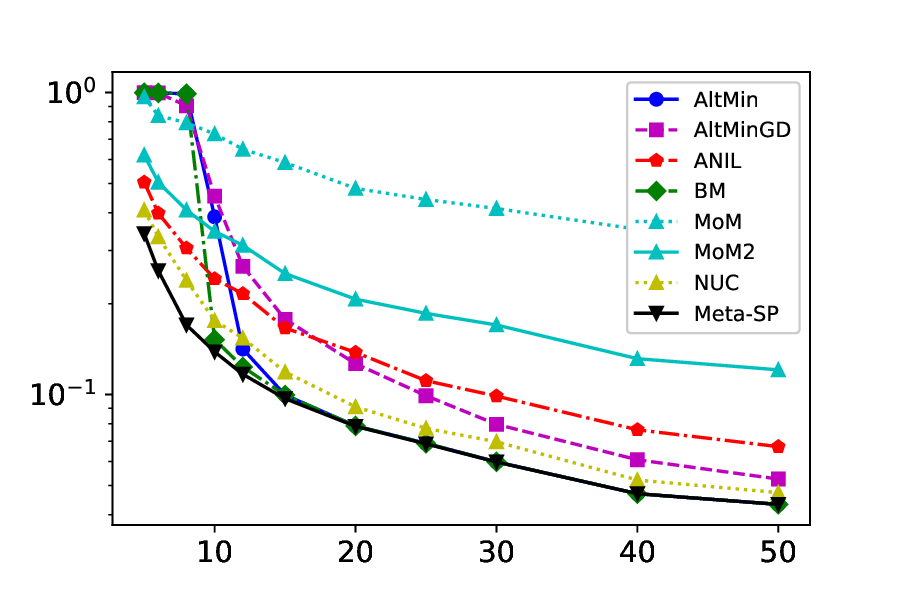}
    \caption{Evolution of $\mathbf{Dist}_1$(left) and $\mathbf{Dist}_2$(right) with the number of tasks $T$ for $s=5$, $T=1600$ and $\sigma=1$.}
    \label{fig_m_T=1600}
\end{figure}

\begin{figure}[htbp!]
    \centering
    \includegraphics[width=.45\textwidth]{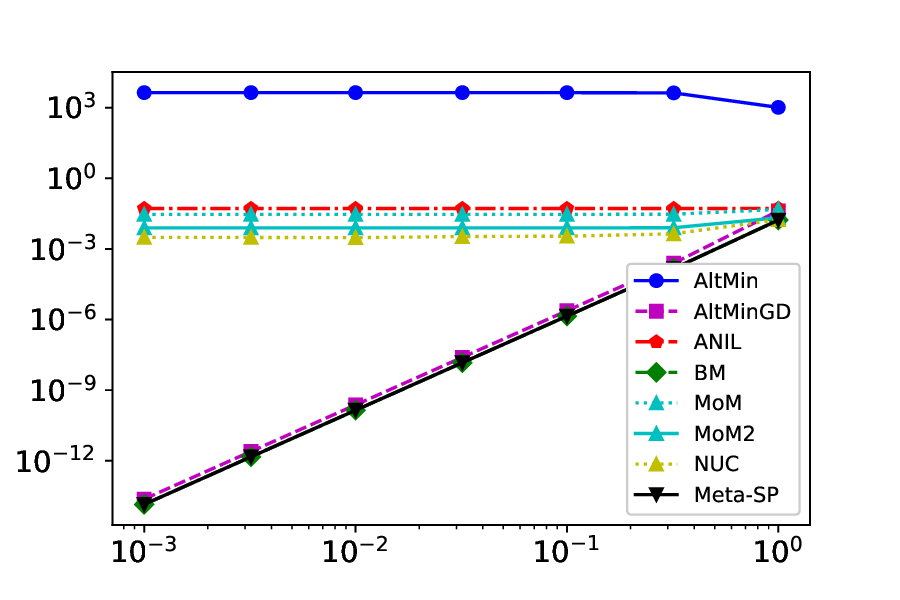}
    \includegraphics[width=.45\textwidth]{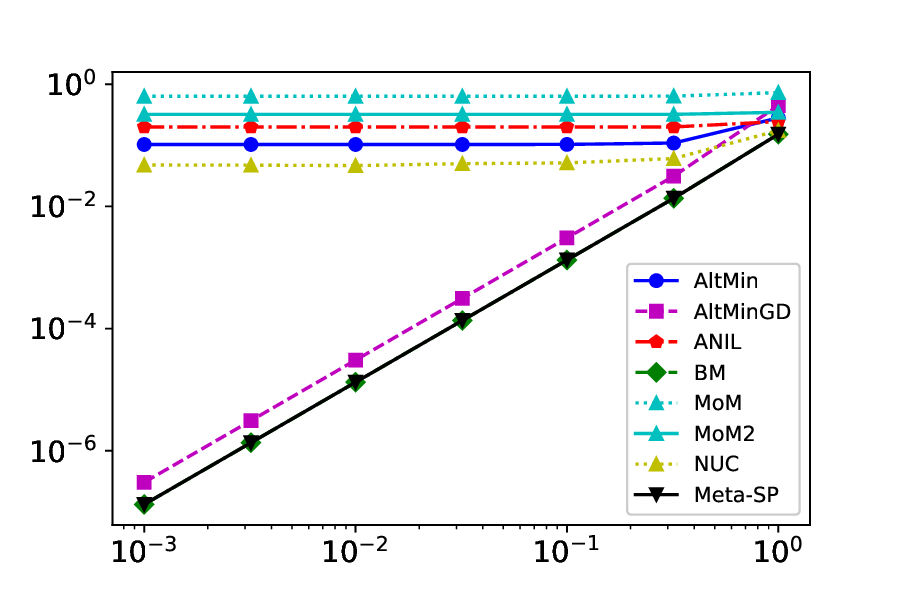}
    \caption{Evolution of $\mathbf{Dist}_1$(left) and $\mathbf{Dist}_2$(right) with variance of noise $\sigma$ for $s=5$, $T=1600$ and $m=10$.}
    \label{fig_sigma_m=10}
\end{figure}

\begin{figure}[htbp!]
    \centering
    \includegraphics[width=.45\textwidth]{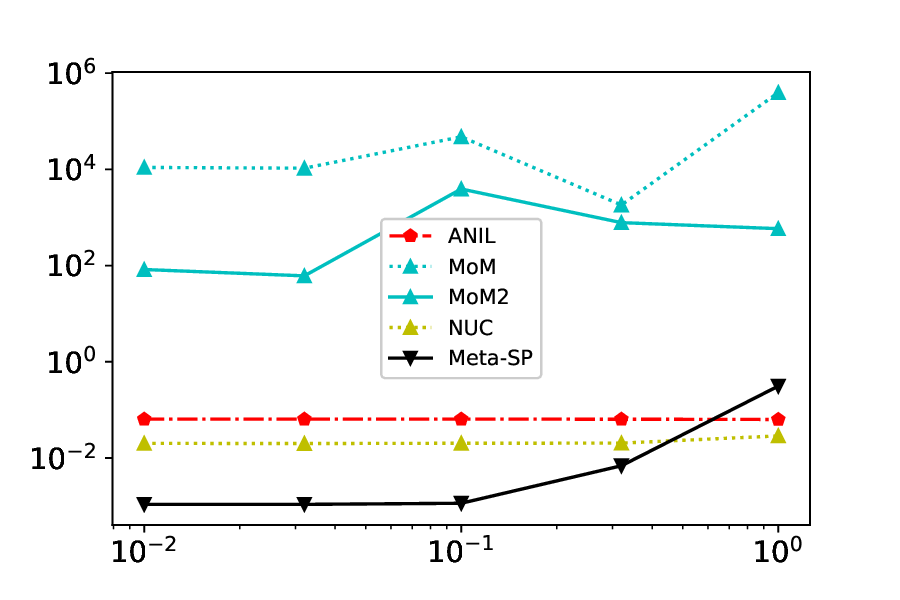}
    \includegraphics[width=.45\textwidth]{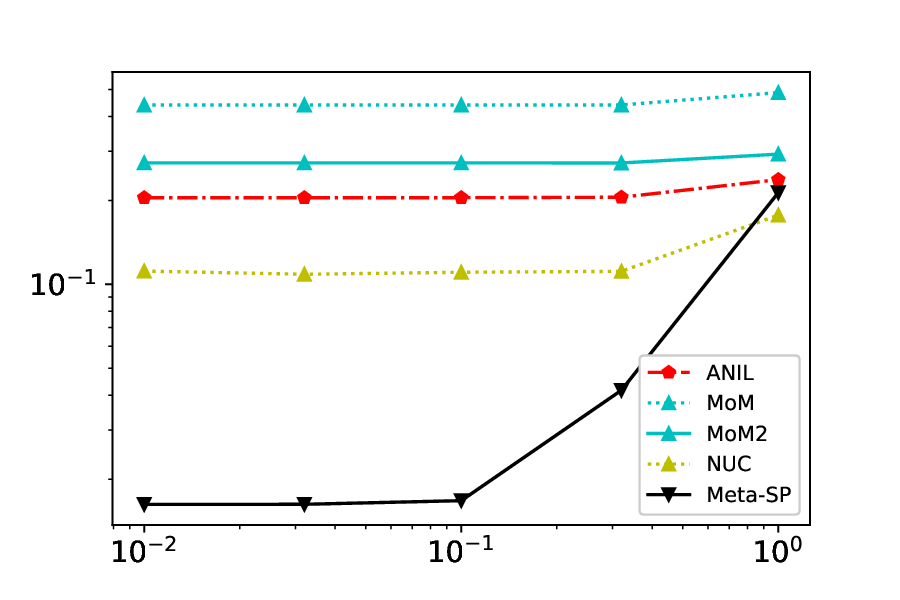}
    \caption{Evolution of $\mathbf{Dist}_1$(left) and $\mathbf{Dist}_2$(right) with variance of noise $\sigma$ for $s=5$, $T=6400$ and $m=5$.}
    \label{fig_sigma_m=5}
\end{figure}

\begin{figure}[htbp!]
    \centering
    \subfigure[Evolution of $\mathbf{Dist}_1$ with iterations.]{
    \label{fig_T=400_m=25_app_1}
    \includegraphics[width=.30\textwidth]{d_r=5_m=25_T=400_F_1.eps}
    \includegraphics[width=.30\textwidth]{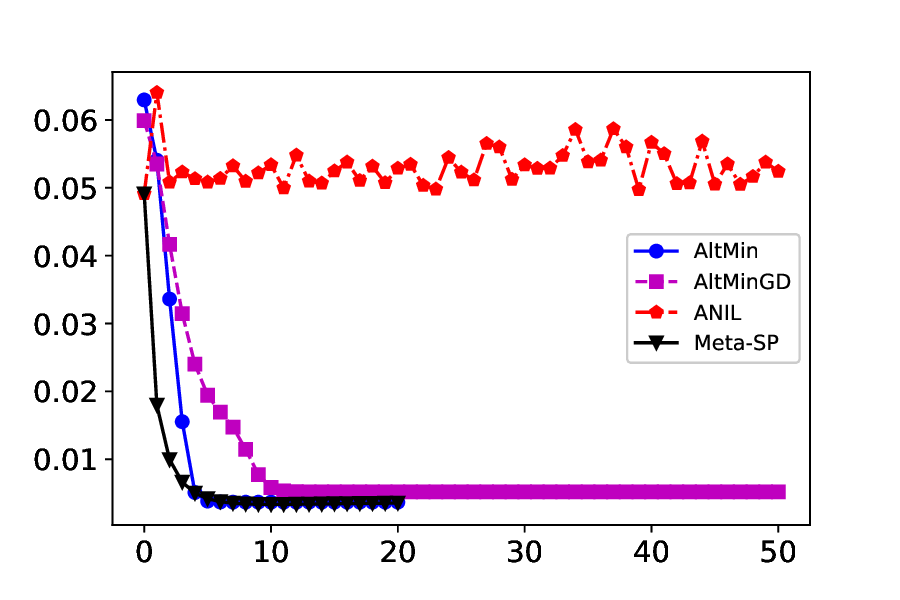}
    \includegraphics[width=.30\textwidth]{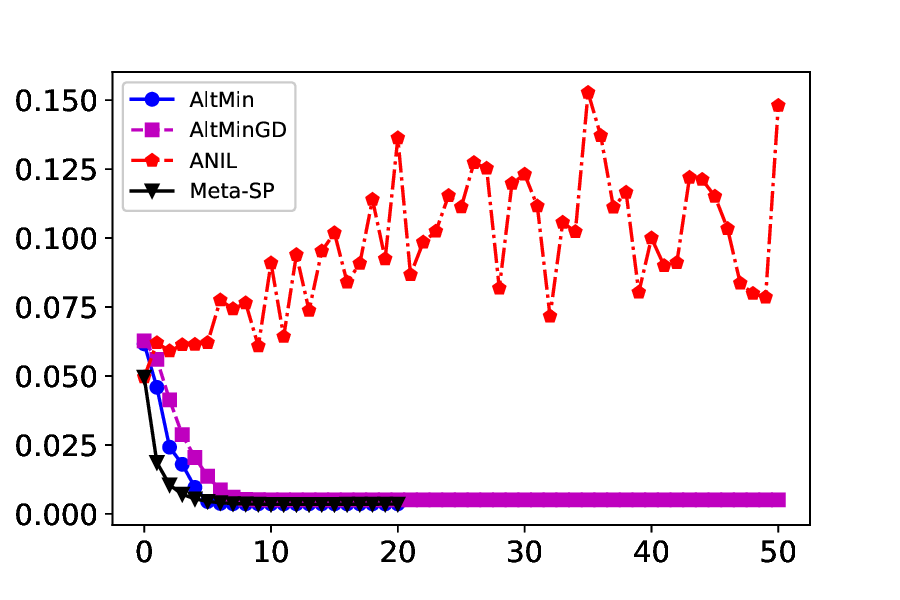}
    }
    \subfigure[Evolution of $\mathbf{Dist}_2$ with iterations.]{
    \label{fig_T=400_m=25_app_2}
    \includegraphics[width=.30\textwidth]{d_r=5_m=25_T=400_sin_1.eps}
    \includegraphics[width=.30\textwidth]{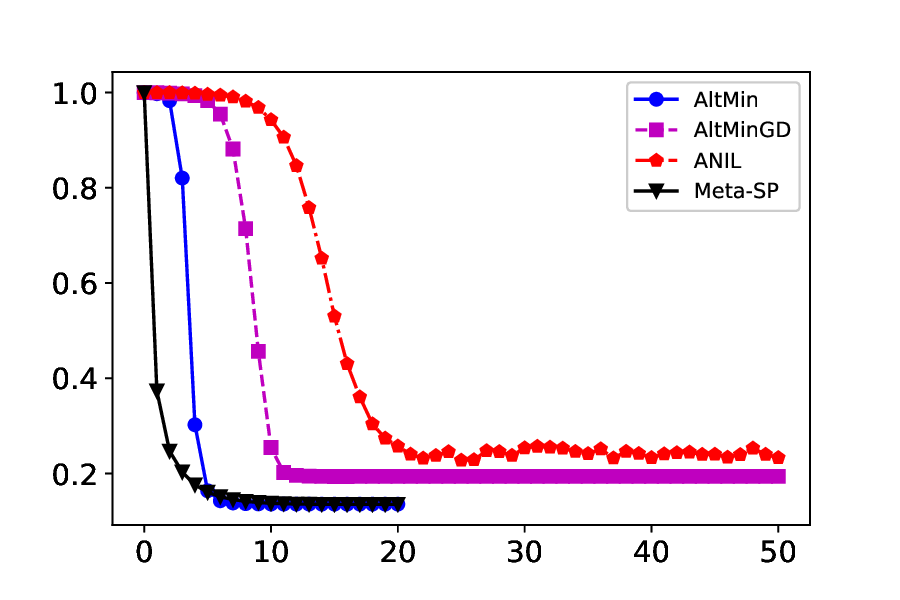}
    \includegraphics[width=.30\textwidth]{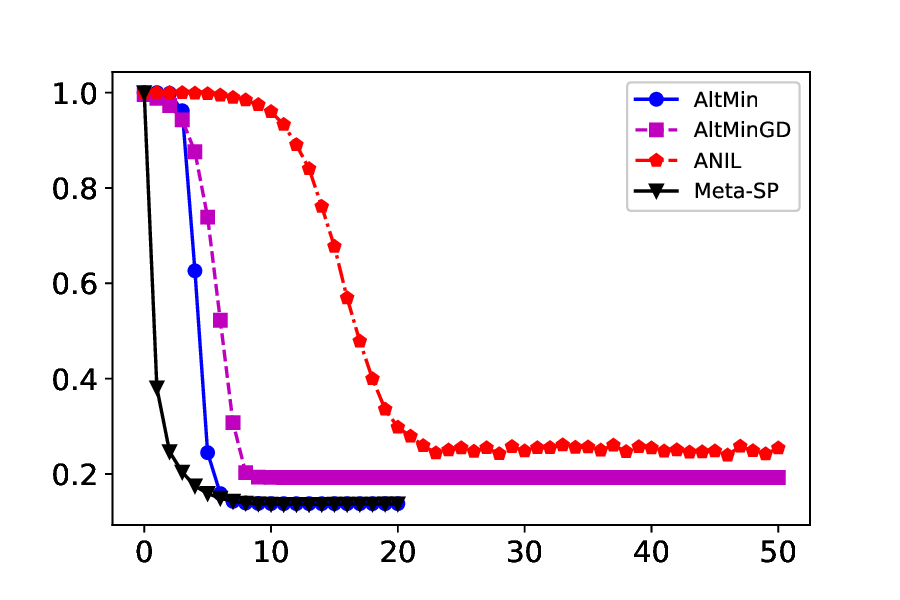}
    }
    \subfigure[Evolution of $\mathbf{Dist}_1$ with computational time.]{
    \label{fig_T=400_m=25_app_3}
    \includegraphics[width=.30\textwidth]{d_r=5_m=25_T=400_Ft_1.eps}
    \includegraphics[width=.30\textwidth]{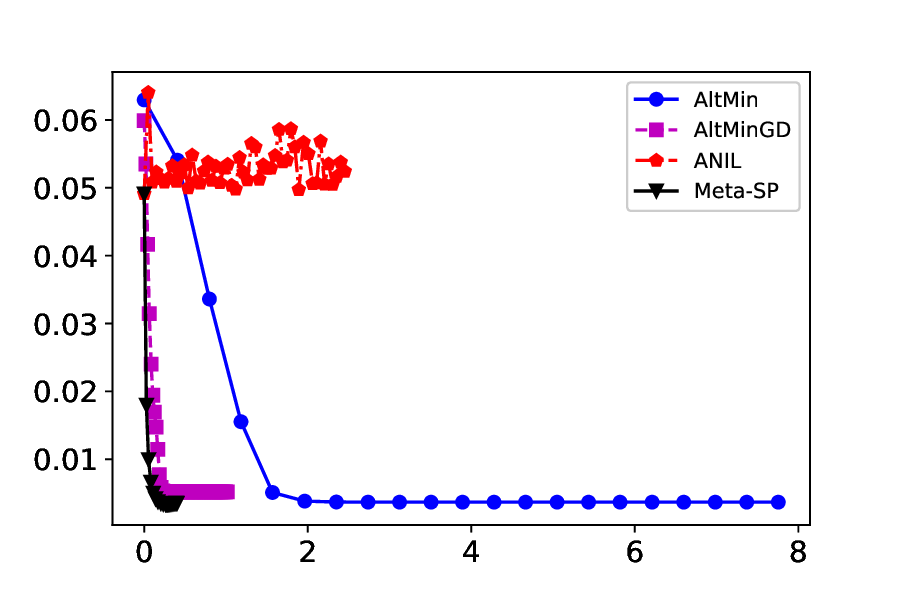}
    \includegraphics[width=.30\textwidth]{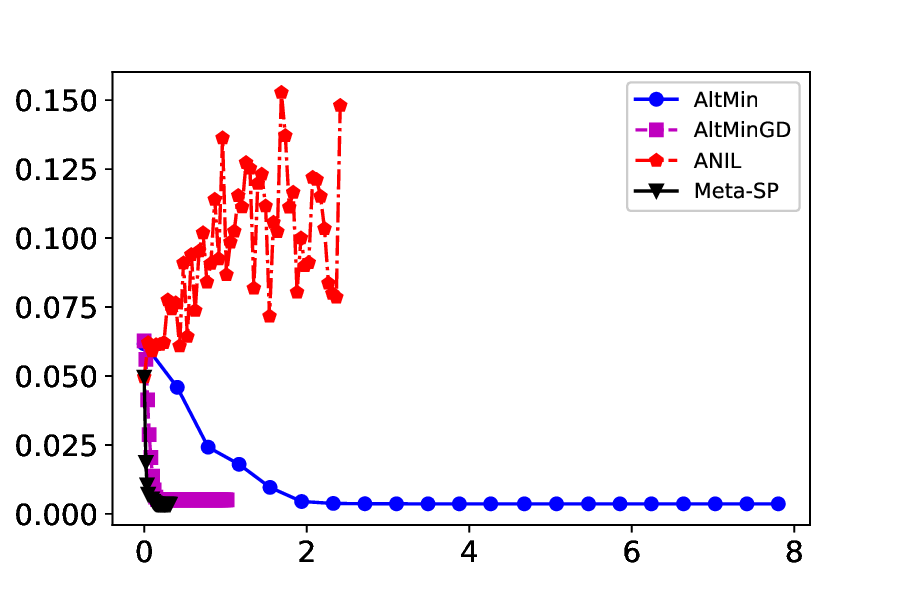}
    }
    \subfigure[Evolution of $\mathbf{Dist}_2$ with computational time.]{
    \label{fig_T=400_m=25_app_4}
    \includegraphics[width=.30\textwidth]{d_r=5_m=25_T=400_sint_1.eps}
    \includegraphics[width=.30\textwidth]{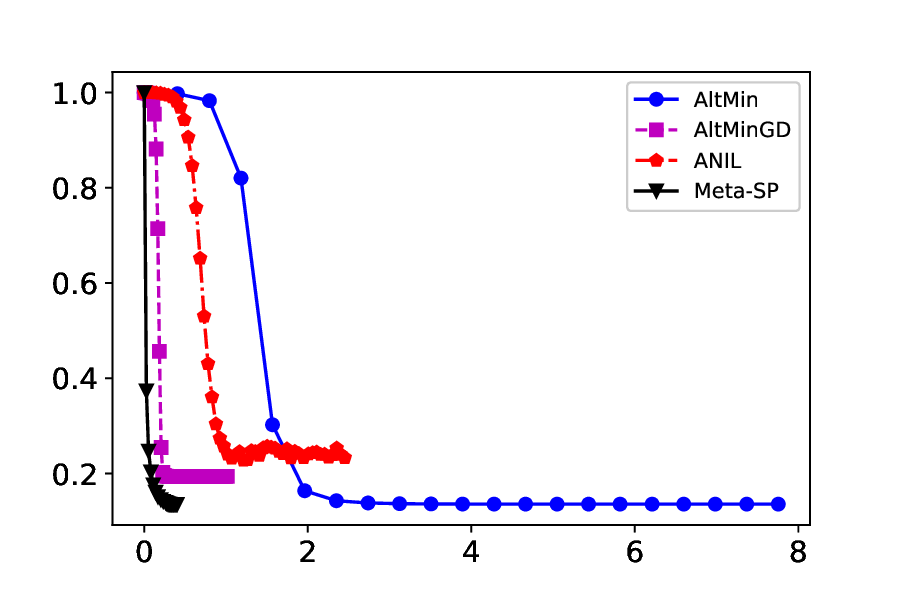}
    \includegraphics[width=.30\textwidth]{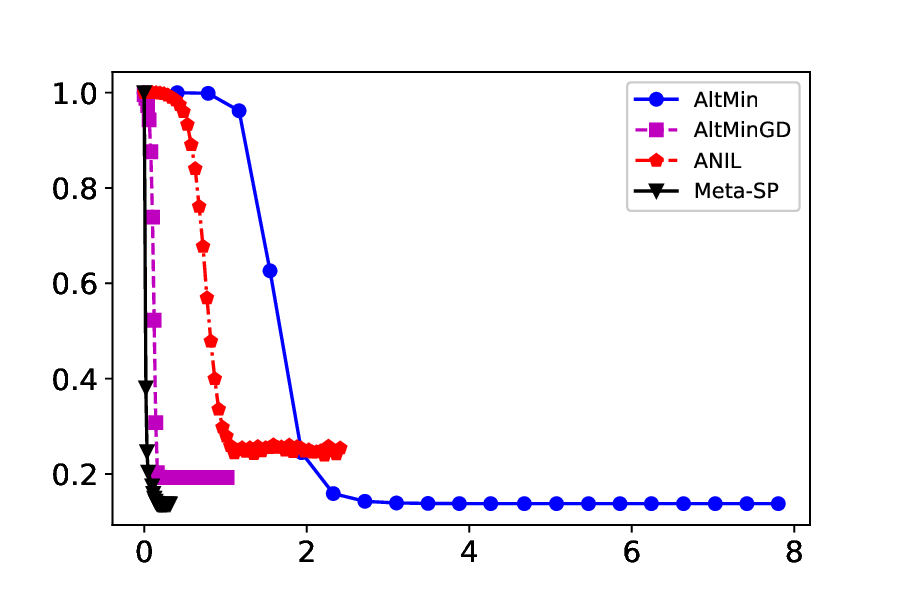}
    }
    \caption{Evolution of $\mathbf{Dist}_1$ and $\mathbf{Dist}_2$ with iterations and computational time (unit: second) for $s=5$, $m=25$, $T=400$ and $\sigma=1$ in three examples.}
    \label{fig_T=400_m=25_app}
\end{figure}

\begin{figure}[htbp!]
    \centering
    \subfigure[Evolution of $\mathbf{Dist}_1$ with iterations.]{
    \label{fig_T=3200_m=10_app_1}
    \includegraphics[width=.30\textwidth]{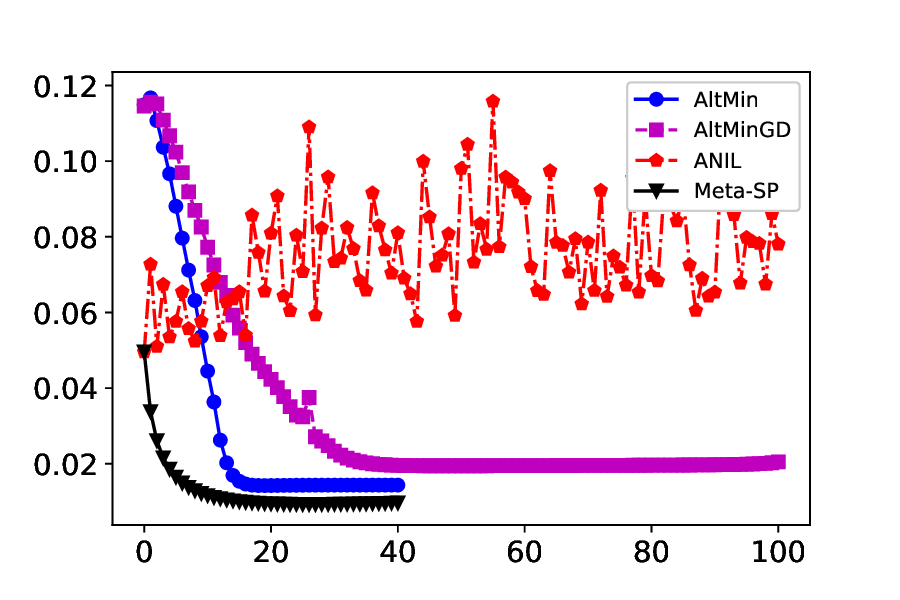}
    \includegraphics[width=.30\textwidth]{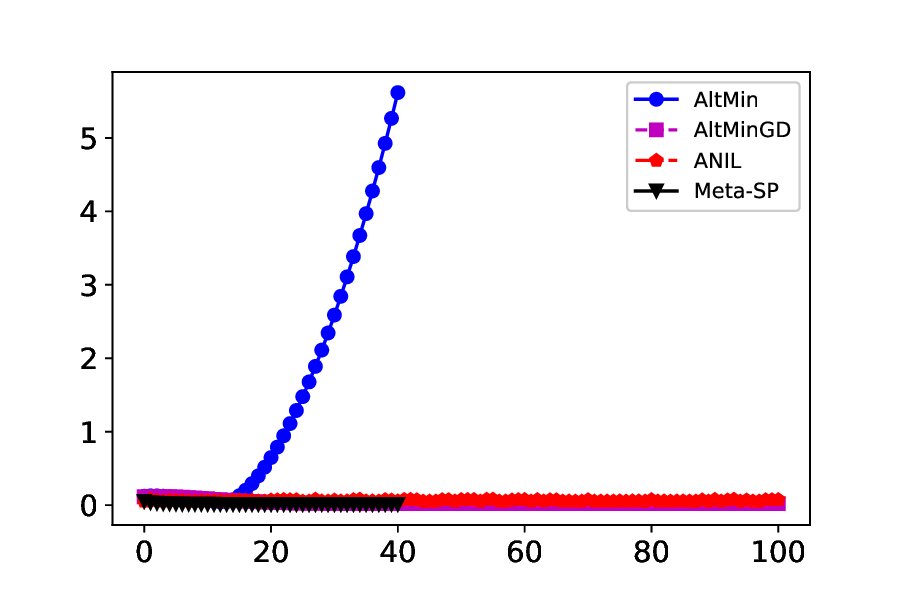}
    \includegraphics[width=.30\textwidth]{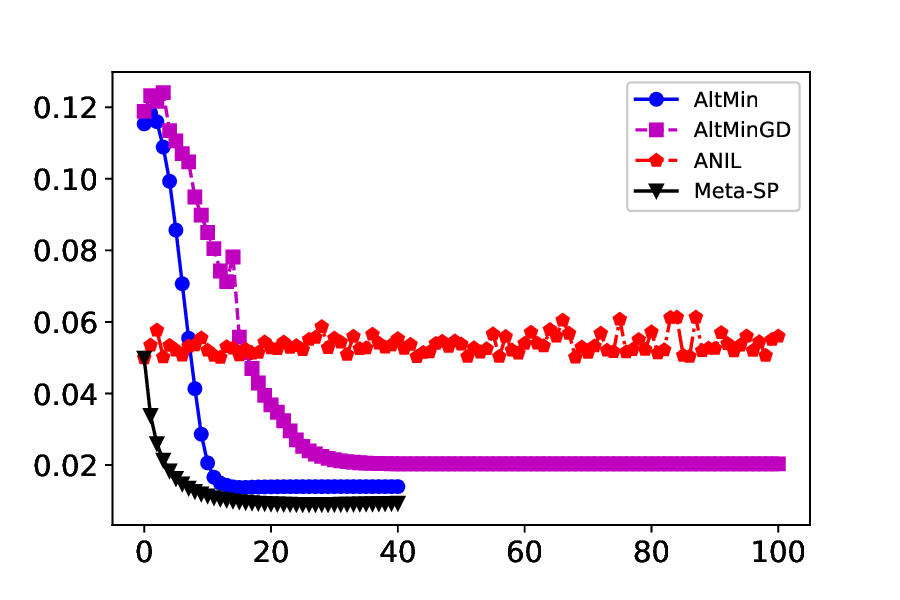}
    }
    \subfigure[Evolution of $\mathbf{Dist}_2$ with iterations.]{
    \label{fig_T=3200_m=10_app_2}
    \includegraphics[width=.30\textwidth]{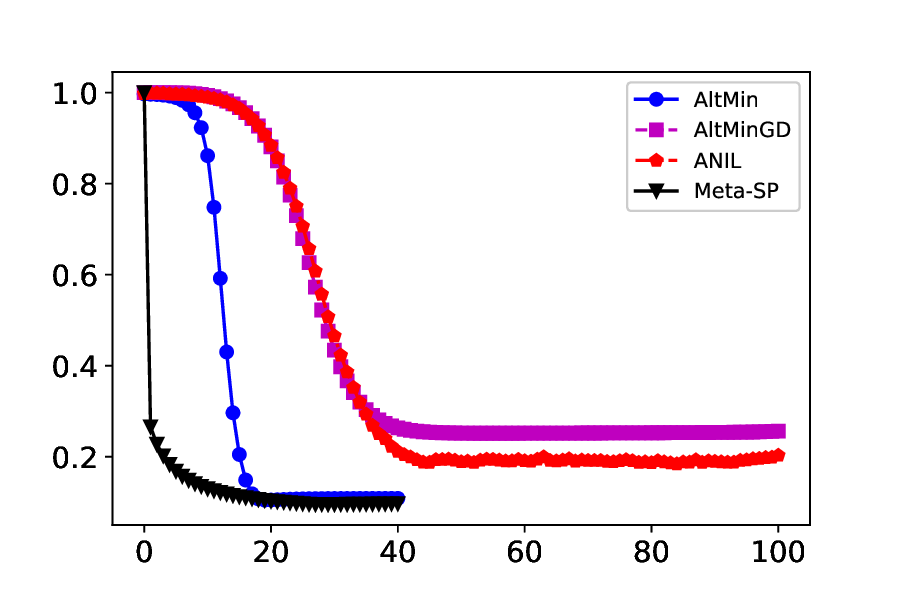}
    \includegraphics[width=.30\textwidth]{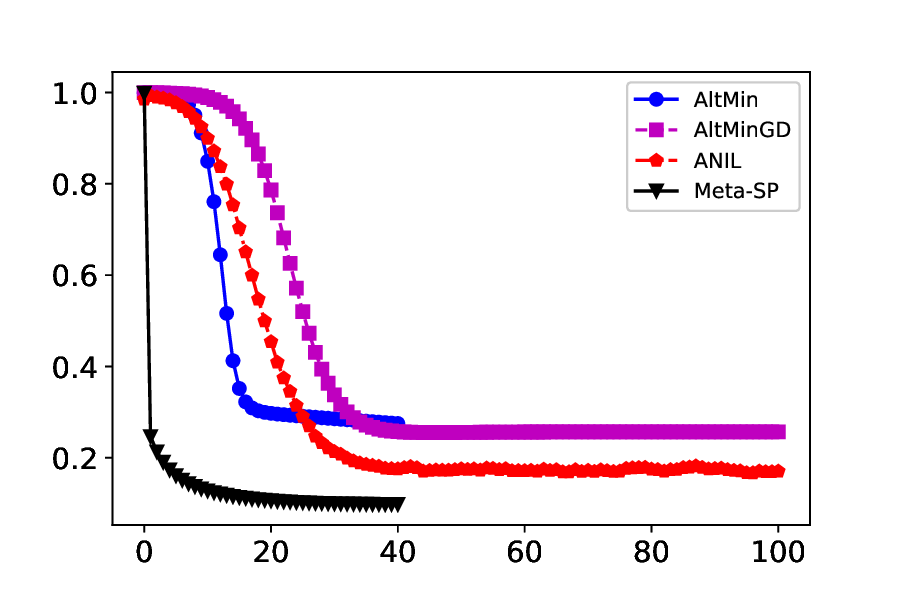}
    \includegraphics[width=.30\textwidth]{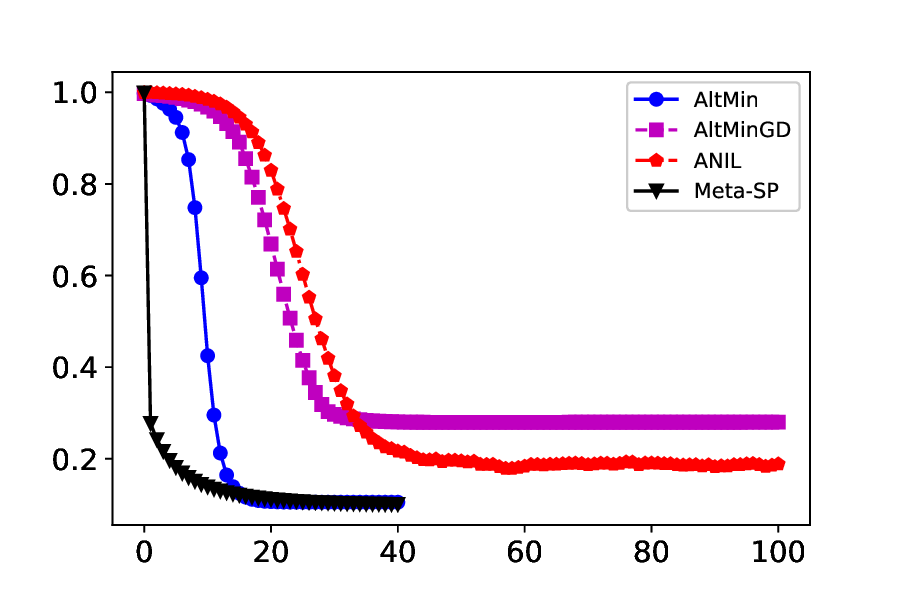}
    }
    \subfigure[Evolution of $\mathbf{Dist}_1$ with time.]{
    \label{fig_T=3200_m=10_app_3}
    \includegraphics[width=.30\textwidth]{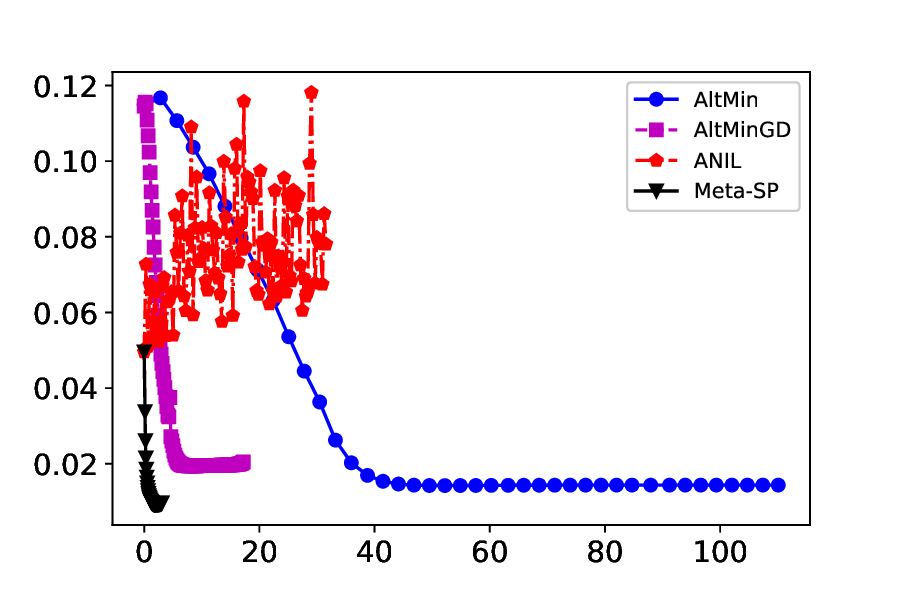}
    \includegraphics[width=.30\textwidth]{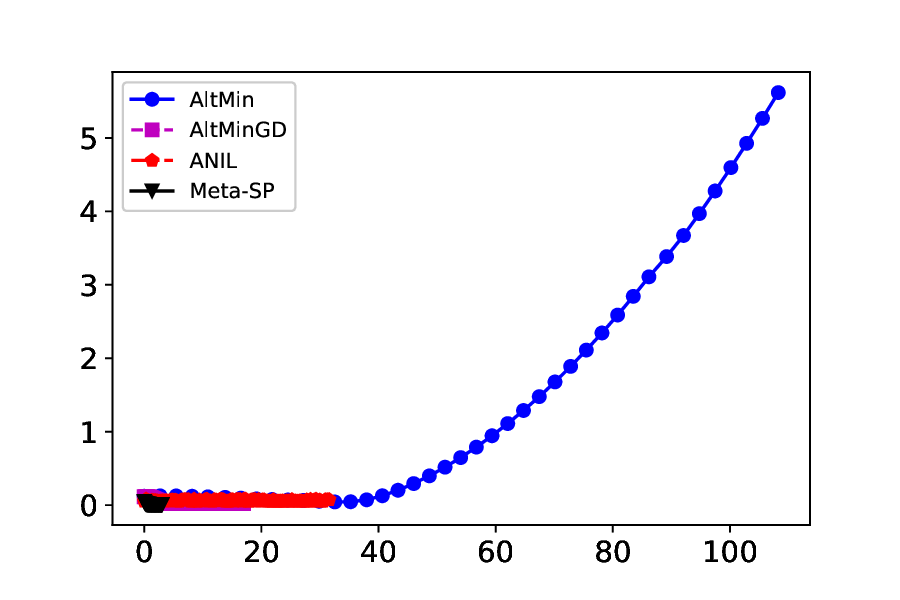}
    \includegraphics[width=.30\textwidth]{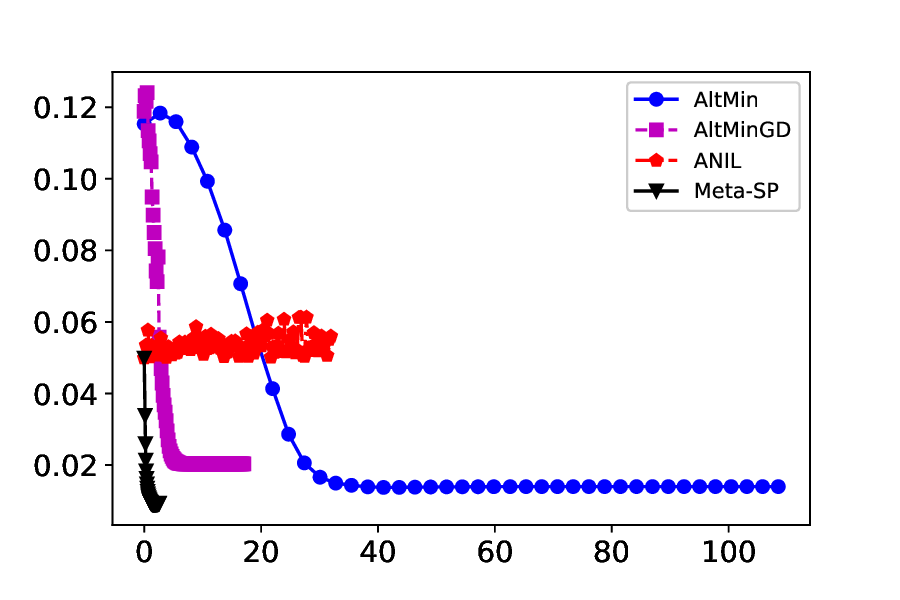}
    }
    \subfigure[Evolution of $\mathbf{Dist}_2$ with computational time.]{
    \label{fig_T=3200_m=10_app_4}
    \includegraphics[width=.30\textwidth]{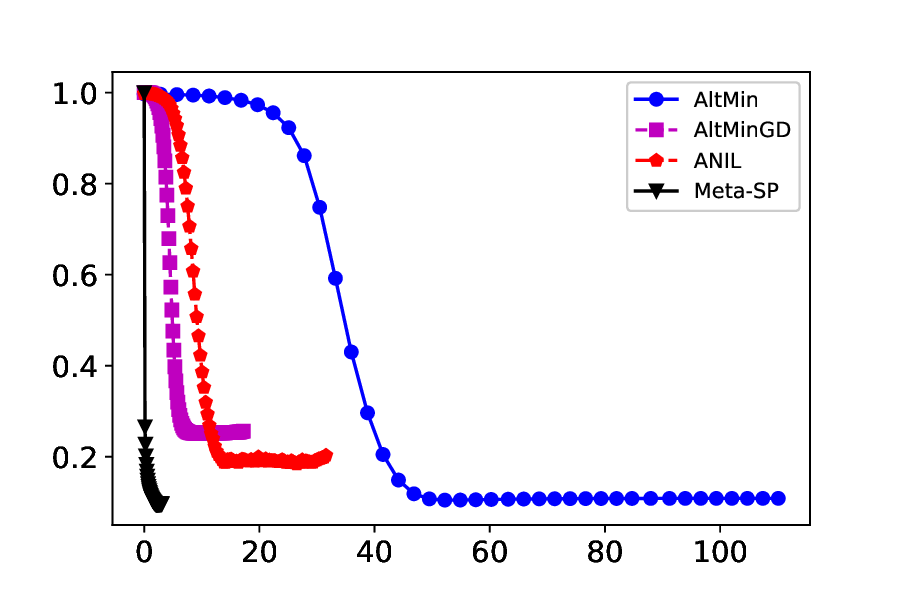}
    \includegraphics[width=.30\textwidth]{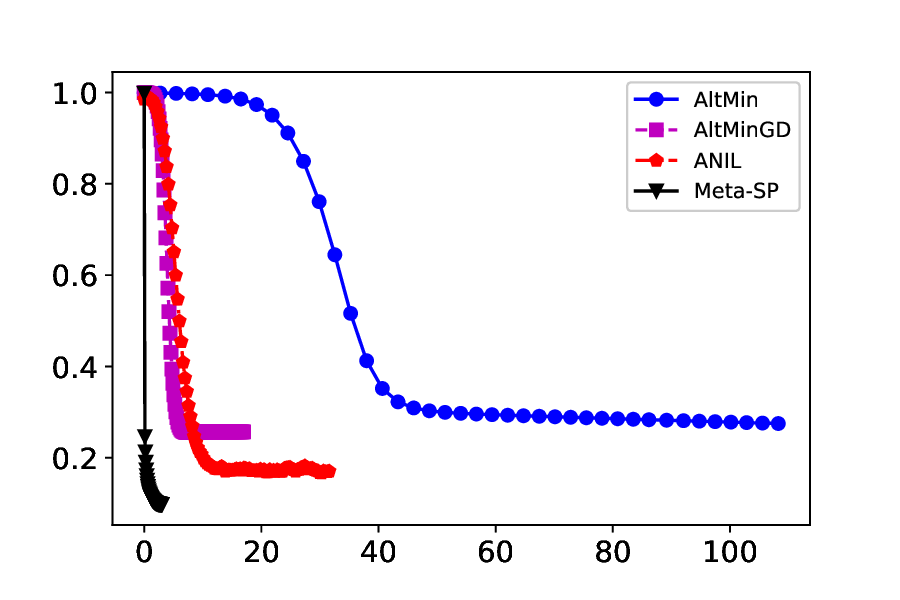}
    \includegraphics[width=.30\textwidth]{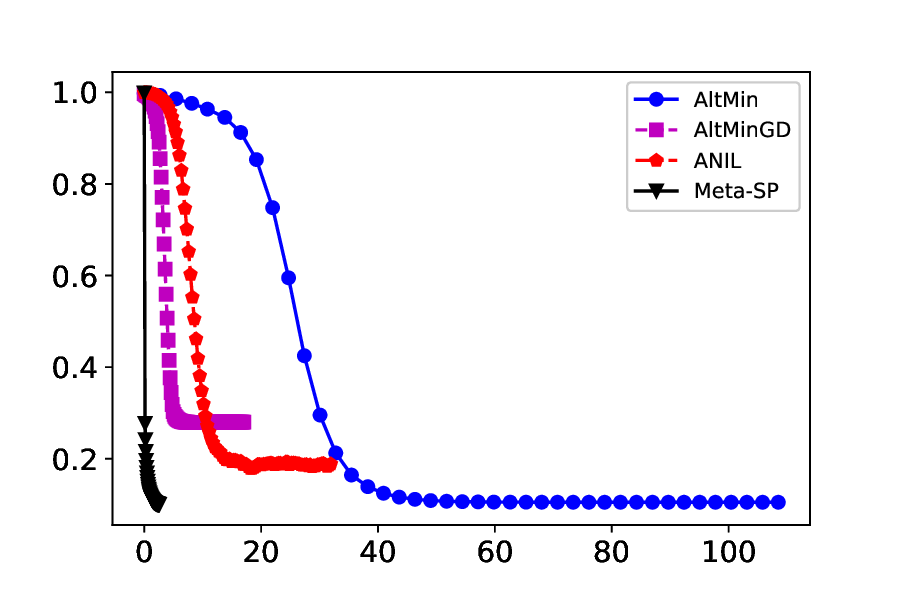}
    }
    \caption{Evolution of $\mathbf{Dist}_1$ and $\mathbf{Dist}_2$ over iterations and computational time (unit: second) for $s = 5$, $m = 10$, $T = 3200$ and $\sigma = 1$ in three examples.}
    \label{fig_T=3200_m=10_app}
\end{figure}

The above experiments investigate the influence of both $m$ and $T$ on the results. These two parameters determine the amount of data utilized to address the problems. All of these experiments consistently demonstrate that our method, \our{}, requires less data to solve the same problem and consistently achieves superior results under the same conditions.

We have presented results pertaining to the parameter $\sigma$ in Figure \ref{fig_sigma_m=25}. Figure \ref{fig_sigma_m=10} illustrates the scenario when $m=10$ and $T=1600$, while Figure \ref{fig_sigma_m=5} covers the case when $m=5$ and $T=6400$. Notably, \texttt{AltMin}, \texttt{AltMinGD}, and \texttt{BM} are omitted in Figure \ref{fig_sigma_m=5} because they cannot effectively function when $m=5$. In Figure \ref{fig_sigma_m=10}, the results closely resemble those in Figure \ref{fig_sigma_m=25}, with the exception that \texttt{AltMin} exhibits less effectiveness. This observation aligns with the findings in Figures \ref{fig_T_m=10} and \ref{fig_m_T=1600}. The outcomes presented in Figure \ref{fig_sigma_m=5} underscore that our method, \our{}, stands out as the sole approach capable of performing effectively in scenarios with extremely limited data represented by small values of $m$.

Figure \ref{fig_T=400_m=25} serves as an illustrative example of how the metrics evolve with iterations and time. To provide further insights, we present three additional examples under the same conditions of $m=25$ and $T=400$ in Figure \ref{fig_T=400_m=25_app}. Furthermore, in Figure \ref{fig_T=3200_m=10_app}, we replicate the experiments with $m=10$ and $T=3200$. Notably, \texttt{AltMin} exhibits suboptimal performance in the second example, signifying its instability when dealing with small values of $m$. Collectively, these experiments underscore that our method, \our{}, not only delivers effective results but also achieves efficiency by necessitating fewer iterations and less time to yield favorable outcomes.

\subsection{Experimental Details} \label{app-exp-detail}
In all the experiments, we generate the true task-invariant matrix $\bB^{\ast}$ by first QR factorizing of a $d \times d$ matrix with elements sampled i.i.d. from the standard normal $\mathcal{N} (0, 1)$, and then retrieving the first $s$ columns. The elements of each task-varying coefficient vector $\bw_t^\ast$ are also sampled i.i.d. from the standard normal $\mathcal{N} (0,1)$. For all methods except \our{}, initialization is based on random draws from the standard normal as well. In contrast, \our{} initializes $\btheta_{t}^{(0)}$ as the zero vector for all $t \in [T]$. For the algorithm \texttt{NUC}, a hyperparameter $\lambda$ called the ``regularization coefficient'' is determined as $\lambda = \frac{\sigma}{T} \sqrt{\frac{T + d^2 / m}{mT}}$, following the recommendation in \citep{boursier2022trace}. When implementing the \texttt{BM} and \texttt{NUC} algorithms, as in \citep{tripuraneni2021provable} and \citep{boursier2022trace}, we use the L-BFGS algorithm \citep{liu1989limited} within the \textbf{scipy} package. Moreover, we employ the \textbf{autograd} package for gradient computations in these two methods.

The choice of step sizes and maximum iterations for our proposed \our{}, as well as \texttt{AltMin}, \texttt{AltMinGD}, and \texttt{ANIL}, are determined empirically. These parameters are set differently for various experiments. In Figure \ref{fig_T_m=25}, the step sizes for \our{} and \texttt{AltMinGD} are 0.25 and 1.0, respectively, while the inner and outer loop step sizes for \texttt{ANIL} are both 0.5. \our{}, \texttt{AltMinGD}, \texttt{AltMin} and \texttt{ANIL} are run for 40, 200, 20, 200 iterations, respectively.
In Figure \ref{fig_T_m=5}, Figure \ref{fig_m_T=800}, Figure \ref{fig_T_m=25_s=25} and Figure \ref{fig_m_T=1600}, the step sizes for \our{} and \texttt{AltMinGD} are 0.05 and 0.2, respectively, while the inner and outer loop step sizes for \texttt{ANIL} are both 0.1. \our{}, \texttt{AltMinGD}, \texttt{AltMin} and \texttt{ANIL} are run for 200, 1000, 100, 800 iterations, respectively.
In Figure \ref{fig_sigma_m=25}, the step sizes for \our{} and \texttt{AltMinGD} are 0.5 and 1.0, respectively, while the inner and outer loop step sizes for \texttt{ANIL} are both 0.5. \our{}, \texttt{AltMinGD}, \texttt{AltMin} and \texttt{ANIL} are run for 300, 200, 100, 400 iterations, respectively.
In Figure \ref{fig_T=400_m=25} and \ref{fig_T=400_m=25_app}, the step sizes for \our{} and \texttt{AltMinGD} are 0.5 and 1.0, respectively, while  the inner and outer loop step sizes for \texttt{ANIL} are both 0.5. \our{}, \texttt{AltMinGD}, \texttt{AltMin} and \texttt{ANIL} are run for 20, 50, 20, 50 iterations, respectively.
In Figure \ref{fig_T_m=10} and Figure \ref{fig_T_m=40_s=25}, the step sizes for \our{} and \texttt{AltMinGD} are 0.1 and 0.4, respectively, while  the inner and outer loop step sizes for \texttt{ANIL} are both 0.2. \our{}, \texttt{AltMinGD}, \texttt{AltMin} and \texttt{ANIL} are run for 100, 500, 50, 400 iterations, respectively.
In Figure \ref{fig_T_m=25_sigma=0}, the step sizes for \our{} and \texttt{AltMinGD} are 0.25 and 1.0, respectively, while  the inner and outer loop step sizes for \texttt{ANIL} are both 0.5. \our{}, \texttt{AltMinGD}, \texttt{AltMin} and \texttt{ANIL} are run for 1000, 200, 100, 400 iterations, respectively. In addition, the regularization coefficient for \texttt{NUC} is set to $1 \times 10^{-5}$ since the best value is zero when it follows the formulation in \cite{tripuraneni2021provable}.
In Figure \ref{fig_sigma_m=10}, the step sizes for \our{} and \texttt{AltMinGD} are 0.3 and 0.4, respectively, while  the inner and outer loop step sizes for \texttt{ANIL} are both 0.2. \our{}, \texttt{AltMinGD}, \texttt{AltMin} and \texttt{ANIL} are run for 15000, 500, 400, 500 iterations, respectively.
In Figure \ref{fig_T=3200_m=10_app}, the step sizes for \our{} and \texttt{AltMinGD} are 0.2 and 0.4, respectively, while  the inner and outer loop step sizes for \texttt{ANIL} are both 0.3. \our{}, \texttt{AltMinGD}, \texttt{AltMin} and \texttt{ANIL} are run for 40, 100, 40, 100 iterations, respectively.

To enhance the reliability of our results, all the results in Figure \ref{fig_T_m=25} to Figure \ref{fig_sigma_m=25} and Figure \ref{fig_T_m=10} to Figure \ref{fig_sigma_m=5} are averaged over 5 trials. The results presented in Figure \ref{fig_mT} and Table \ref{tab_mT} are derived from experimenting with different $T$ values for different $m$, approximated to the nearest hundred or ten, and represent the average error over 5 trials for each scenario.

\end{appendices}

%%%% Bibliography  %%%%%%%%%%
\bibliographystyle{plain}
\bibliography{meta.bib}

\end{document}